\documentclass[11pt]{article}
\usepackage{geometry}
\usepackage{amsmath, amssymb, amsthm}
\usepackage{bm, graphicx}
\usepackage[colorlinks=true]{hyperref}
\usepackage[backend=biber]{biblatex}
\usepackage{algorithm}
\usepackage{algpseudocode}
\usepackage{booktabs}
\usepackage{pifont}
\usepackage{multirow}
\usepackage{tikz}
\usetikzlibrary{arrows.meta, positioning, decorations.pathreplacing}
\usepackage{enumitem}

\newtheorem{theorem}{Theorem}[section]
\newtheorem{proposition}[theorem]{Proposition}
\newtheorem{definition}[theorem]{Definition}
\newtheorem{remark}[theorem]{Remark}

\newtheorem{example}[theorem]{Example}

\title{
The Geometry of Memorization: Finite-Time Spectral Sensitivity as a Diagnostic for Flow Matching Models
}

\author{Shuchan Wang}
\date{\today}

\begin{document}
\maketitle

% ================================================================
\begin{abstract}
Continuous-time generative frameworks construct probability paths between base and target domains by optimizing time-dependent velocity fields. While theoretical targets favor straight trajectories, empirical networks develop complex path deformations. This paper presents the \emph{Finite-Time Spectral Sensitivity (FTSS)} $g(t)$, a gradient-free, forward-pass metric that exposes flow geometry by tracking the root-mean-square singular value of the state-transition matrix. Serving as a continuous proxy for stable rank, $g(t)$ reveals a distinct geometric pathology under data scarcity: while generalizing models maintain stable effective dimensions, overfitting causes a spectral collapse. We leverage this structural phenomenon to develop an internal geometric audit based on $g(t)$. Our framework detects generative memorization using purely internal trajectory dynamics, removing the need for external membership queries or baseline data comparison.
\end{abstract}

\section{Introduction}
\label{sec:intro}

Continuous-time frameworks like Flow Matching~\cite{lipmanflow} and Flow Straight and Fast~\cite{liuflow} optimize a velocity field $v_\theta(x,t)$ to construct a diffeomorphism between the base and target domains. While optimal transport theory yields straight-line trajectories at global optimality, empirically trained neural networks exhibit non-linear paths influenced by model architectures and datasets. Diagnosing how these empirical vector fields evolve throughout the generative time horizon remains difficult due to a lack of scalable geometric diagnostics capable of measuring local directional sensitivity and volume deformation.

\subsection{The Geometrical Pathology of Memorization}
Monitoring these internal dynamics is critical for identifying generative memorization, where high-capacity models overfit by collapsing continuous latent spaces onto discrete training samples~\cite{somepalli2023diffusion, carlini2023extracting}. Existing diagnostics are predominantly extrinsic and data-dependent, relying on post-hoc nearest-neighbor searches against the training pool or membership inference attacks \cite{shokri2017membership, feldman2020neural,zhang2017understanding}. We propose that memorization can be detected as an \emph{intrinsic} structural property of the transport map. When a model overfits, it squashes orthogonal variations to force trajectories onto a subset of learned targets, altering the underlying spectral properties of the flow. Detecting this structural collapse removes the need for training data access during auditing.

\subsection{Computational Complexity of Geometric Diagnostics}
Tracking the geometry of deep networks via exact spatial derivatives is computationally prohibitive. For an ambient dimension $d$, constructing the dense spatial Jacobian requires $d$ backward passes via automatic differentiation. As detailed in Table~\ref{tab:computational_cost}, this presents an insurmountable computational barrier for high-dimensional image models~\cite{chen2018neural, grathwohlffjord, griewank2008evaluating, hutchinson1989stochastic, dinh2017density}. While Jacobian-based diagnostics are used to study generalization in classification networks~\cite{novak2018sensitivity}, explicitly constructing these operators for continuous-time generative models does not scale. This highlights the need for a proxy metric that captures essential spectral and volume properties without explicit matrix computation.

\begin{table}[htbp]
\centering
\small 
\setlength{\tabcolsep}{4pt} 
\caption{Computational complexity and operational costs of computing geometric quantities for continuous-time generative models in $d$ dimensions.}
\label{tab:computational_cost}
\begin{tabular}{@{}llll@{}}
\toprule
Geometric Metric & Exact Formulation & Computational Cost & Mode of Evaluation \\
\midrule
Full Spatial Jacobian & $J_\theta = \nabla_x v_\theta \in \mathbb{R}^{d \times d}$ & $O(d \cdot \text{Cost}(v_\theta))$ & Dense Backpropagation \\
Vector-Jacobian Product & $v^\top J_\theta \in \mathbb{R}^d$ & $O(1 \cdot \text{Cost}(v_\theta))$ & Single Backward Pass \\
Exact Spectrum ($\sigma_i$) & $\text{SVD}(\Phi(1,t))$ & $O(d^3 + d \cdot \text{Cost}(\phi))$ & Prohibitive at Scale \\
\textbf{FTSS (Ours)} & $\hat{g}(t)$ & $O(K \cdot T \cdot \text{Cost}(\phi))$ & \textbf{Gradient-Free Forward Pass} \\
\bottomrule
\end{tabular}
\end{table}

To circumvent this scalability constraint, we introduce the \emph{Finite-Time Spectral Sensitivity (FTSS)} $g(t)$---a gradient-free alternative computed entirely via forward-pass finite differences, bypassing dense operator construction altogether.

\subsection{Summary of Contributions}
The primary contributions of this work are as follows:

\begin{itemize}
    \item \textbf{Mathematical Properties of FTSS:} We formalize the FTSS $g(t)$ in Definition~\ref{def:ftss}. Proposition~\ref{prop:rms} establishes its direct equivalence to the root-mean-square (RMS) singular value of the state-transition matrix, validating it as a continuous proxy for stable rank in Remark~\ref{rm:srank}. Using the arithmetic-geometric mean inequality, we bound the absolute Jacobian determinant pointwise in Proposition~\ref{prop:detbound}, which directly enables the proof of an explicit upper bound on the differential entropy of the generated target distribution in Proposition~\ref{prop:entropy_bound}.
    \item \textbf{Empirical Characterization of Spectral Collapse and Memorization:} Section~\ref{sec:bottleneck_dynamics} demonstrates experimentally that spectral collapse consistently occurs during the late stages of trajectory generation in low-data regimes. Based on this phenomenon, we introduce and validate an empirical Spectral Collapse Ratio $M$. In Section~\ref{sec:diagnostic}, empirical results across multiple architectures demonstrate that $M$ accurately identifies overfitted continuous-time models without requiring training pool access or baseline data comparisons.
\end{itemize}

\section{Finite-Time Spectral Sensitivity}
\label{sec:gain}

To quantify the dimensional pinching that characterizes memorization, we now 
formalize the FTSS, denoted $g(t)$. This metric adapts the core philosophy of 
Finite-Time Lyapunov Exponents---tracking perturbation growth over finite time 
horizons---but aggregates across the full singular value spectrum rather than 
isolating only the dominant mode (see Appendix~\ref{app:lyapunov} for a 
detailed comparison). By measuring the root-mean-square variational deformation 
of the tangent space, FTSS provides a scalar summary of how severely the flow 
compresses or expands volume across all directions simultaneously.

This section formalizes the FTSS as an intrinsic geometric observable for continuous-time generative models. We establish its foundation via the linearized variational equations of the vector field and provide its interpretation in terms of the spectrum of the forward state-transition matrix. We then characterize its relationship with volume forms and a differential entropy bound. Finally, we present a computationally efficient, gradient-free Monte Carlo estimator for this quantity requiring only forward-pass execution.

\subsection{Mathematical Foundations and Variational Dynamics}
\label{sec:definition}

Let $v_\theta: \mathbb{R}^d \times [0,1] \to \mathbb{R}^d$ denote a parameterized time-dependent velocity field defining a continuous-time generative process. For an initial latent state $x_0 \sim p_0(x)$, the trajectory $x_t$ is governed by the ordinary differential equation (ODE):
\begin{equation}
\frac{dx_t}{dt} = v_\theta(x_t, t).
\end{equation}
We denote the flow map from time $t$ to time $\tau$ as $\phi_{t\to\tau}: \mathbb{R}^d \to \mathbb{R}^d$, such that $\phi_{t\to\tau}(x_t) = x_\tau$. To analyze the stability of this map, we evaluate the propagation of an infinitesimal perturbation $\delta_t \in T_{x_t}\mathbb{R}^d$ introduced at intermediate time $t$. The time evolution of this perturbation is governed by the linearized variational equation along the trajectory:
\begin{equation}
\frac{d}{ds}\delta_s = J_\theta(x_s, s)\,\delta_s, \quad s \in [t,1],
\end{equation}
where $J_\theta(x,s) = \nabla_x v_\theta(x,s) \in \mathbb{R}^{d \times d}$ is the spatial Jacobian of the velocity field. The solution to this linear non-autonomous system is given by the pushforward action of the flow map, encapsulated by the state-transition matrix $\Phi(1,t; x_t) \equiv D\phi_{t\to 1}(x_t)$, which satisfies:
\begin{equation}
\frac{d}{ds}\Phi(s,t; x_t) = J_\theta(x_s,s)\,\Phi(s,t; x_t), \quad \Phi(t,t; x_t) = I.
\end{equation}
Consequently, the local displacement at the final time horizon $t=1$ conforms to the first-order relation:
\begin{equation}\label{eq:var}
\phi_{t\to 1}(x_t + \delta_t) - \phi_{t\to 1}(x_t) = \Phi(1,t; x_t)\,\delta_t + O(\|\delta_t\|^2).
\end{equation}

To capture the average directional sensitivity of the terminal generative state to intermediate perturbations, we define FTSS over isotropic variations.

\begin{definition}[Finite-Time Spectral Sensitivity]
\label{def:ftss}
Given a flow map $\phi_{t\to 1}$ and an intermediate state distribution $p_t(x_t)$, the FTSS $g: [0,1] \to \mathbb{R}_{\ge 0}$ is defined as:
\begin{equation}
g(t) = \mathbb{E}_{x_t \sim p_t}\left[\lim_{\varepsilon \to 0}\; \mathbb{E}_{u \sim \mathrm{Unif}(S^{d-1})}\left[ \frac{\|\phi_{t\to 1}(x_t + \varepsilon u) - \phi_{t\to 1}(x_t)\|}{\varepsilon} \right]\right],
\end{equation}
where $S^{d-1}$ is the unit sphere in $\mathbb{R}^d$.
\end{definition}

By substituting the variational expansion~\eqref{eq:var} into Definition~\ref{def:ftss}, the inner limit simplifies to the directional norm $\|\Phi(1,t; x_t)u\|$. The baseline value $g(t) = 1$ indicates a state of mean-square spectral conservation, where stretching and compression are balanced on average across all dimensions. This condition naturally holds for perfectly straight identity transport ($\Phi = I$), representing a stable baseline for the overall flow field. Values of $g(t) > 1$ signify local trajectory amplification, while $g(t) < 1$ denotes net compression.

We now show that this directional expectation captures the spectral magnitude of the underlying transport operator.

\begin{proposition}[RMS Singular Value Interpretation]
\label{prop:rms}
Let $\sigma_1(t; x_t) \ge \sigma_2(t; x_t) \ge \dots \ge \sigma_d(t; x_t) \ge 0$ be the singular values of the forward state-transition matrix $\Phi(1,t; x_t)$. Then the FTSS satisfies:
\begin{equation}
g(t) = \frac{1}{\sqrt{d}}\;\mathbb{E}_{x_t}\Bigl[\|\Phi(1,t; x_t)\|_F\Bigr],
\end{equation}
and its square yields the expected mean-square singular value:
\begin{equation}
g(t)^2 = \frac{1}{d}\;\mathbb{E}_{x_t}\!\left[\sum_{i=1}^{d} \sigma_i^2(t; x_t)\right].
\end{equation}
\end{proposition}

\begin{proof}
For a fixed trajectory point $x_t$, let $A = \Phi(1,t; x_t) \in \mathbb{R}^{d \times d}$. Evaluating the directional variance of the linear map over the uniform spherical measure yields:
\[
\mathbb{E}_{u \sim \mathrm{Unif}(S^{d-1})}[\|Au\|^2] = \mathbb{E}_u[u^\top A^\top A u] = \mathrm{Tr}\bigl(A^\top A\,\mathbb{E}_u[uu^\top]\bigr).
\]
By the rotational invariance of the uniform distribution on $S^{d-1}$, the second-moment matrix is isotropic: $\mathbb{E}_u[uu^\top] = \frac{1}{d}I$. Applying linearity of the trace operator:
\[
\mathbb{E}_u[\|Au\|^2] = \frac{1}{d}\,\mathrm{Tr}(A^\top A) = \frac{1}{d}\,\|A\|_F^2.
\]
Recognizing that the squared Frobenius norm matches the sum of the squared singular values, $\|A\|_F^2 = \sum_{i=1}^d \sigma_i^2(t; x_t)$, taking the expectation over $p_t(x_t)$ yields the specified identity.
\end{proof}

\subsection{Geometric Properties, Volume Forms, and Entropy Bounds}
\label{sec:properties}

While the Jacobian determinant $|\det \Phi(1,t; x_t)| = \prod_{i=1}^d \sigma_i$ tracks global volume contraction or expansion, it is fundamentally blind to highly anisotropic deformations where severe dimensional collapse along one axis is numerically masked by an isolated expansion along another.

\begin{example}[Insensitivity of the Jacobian Determinant]
\label{ex:anisotropic_collapse}
Consider a two-dimensional system ($d=2$) where a trajectory undergoes a sharp compression along its first principal direction such that $\sigma_1 = 10^{-2}$, while simultaneously stretching along the second direction such that $\sigma_2 = 10^2$.

Computing the absolute Jacobian determinant yields:
\[
|\det \Phi(1,t; x_t)| = \sigma_1 \cdot \sigma_2 = 10^{-2} \cdot 10^2 = 1,
\]
falsely signaling perfect volume preservation and structural stability despite a severe loss of dimensionality along the first axis.

In contrast, the FTSS aggregates the spectrum additively:
\[
g(t) = \sqrt{\frac{1}{2}(\sigma_1^2 + \sigma_2^2)} = \sqrt{\frac{1}{2}(10^{-4} + 10^4)} \approx 70.71.
\]
This additive aggregation scales heavily based on the dominant deformation axis, ensuring that the local dimensional collapse is not obscured by the outlying expansion.
\end{example}

This behavior highlights a deeper theoretical link: the local singular value spectrum provides strict analytical control over instantaneous volume deformation. Specifically, via the arithmetic-geometric mean inequality, the absolute Jacobian determinant is pointwise bounded above by the local RMS singular value raised to the dimensional power.

Consequently, when this local RMS value is below unity, it serves as a tighter, more conservative indicator of a contractive regime than the volume form alone. While a single principal axis can technically still stretch up to $\sqrt{d}$ times the local RMS value, the quadratic scaling strictly caps this potential growth, preventing the unconstrained singular value explosions that typically mask dimensional collapse in the standard determinant. 

We formalize this structural coupling in the following pointwise bound:

\begin{proposition}[Pointwise Determinant Bound]
\label{prop:detbound}
Conditioned on any intermediate trajectory state $x_t$, the absolute Jacobian determinant of the remaining forward flow map is bounded above by the mean-square singular values along that specific path:
\begin{equation}
|\det \Phi(1,t; x_t)| \le \left(\frac{1}{d} \sum_{i=1}^d \sigma_i^2(t; x_t)\right)^{d/2}.
\end{equation}
\end{proposition}

\begin{proof}
Let $\sigma_1, \dots, \sigma_d$ be the singular values of $\Phi(1,t; x_t)$ for a fixed trajectory starting at $x_t$. By applying the arithmetic-geometric mean (AM-GM) inequality to the squared singular values, we obtain:
\[
|\det \Phi(1,t; x_t)| = \left(\prod_{i=1}^d \sigma_i^2(t; x_t)\right)^{1/2}
\le \left(\frac{1}{d} \sum_{i=1}^d \sigma_i^2(t; x_t)\right)^{d/2}.
\]
\end{proof}

This pointwise structural coupling allows us to globally bound the differential entropy transformation of the generated data distribution by applying Jensen's inequality to the logarithmic volume.

\begin{proposition}[Entropy-Compression Bound]
\label{prop:entropy_bound}
Let $\phi_{0\to 1}: \mathbb{R}^d \to \mathbb{R}^d$ be a smooth, diffeomorphic transport map pushing forward the base latent distribution $X \sim p_0$ to the target distribution $Y \sim p_1$. The differential entropy $h(Y)$ satisfies the upper bound:
\begin{equation}
h(Y) \le h(X) + d \log g(0).
\end{equation}
\end{proposition}

\begin{proof}
By the standard change-of-variables formula for probability densities under smooth invertible transformations, the target distribution differential entropy expands as:
\[
h(Y) = h(X) + \mathbb{E}_{X}\bigl[\log |\det D\phi_{0\to 1}(X)|\bigr].
\]
Taking the logarithm of the pointwise determinant bound from Proposition~\ref{prop:detbound} at $t=0$ yields:
\[
\log |\det D\phi_{0\to 1}(X)| \le \frac{d}{2} \log\left(\frac{1}{d} \sum_{i=1}^d \sigma_i^2(0; X)\right).
\]
Taking the expectation over $X \sim p_0$ and applying Jensen's inequality for the concave logarithm:
\begin{align*}
\mathbb{E}_{X}\bigl[\log |\det D\phi_{0\to 1}(X)|\bigr]
&\le \frac{d}{2} \mathbb{E}_{X}\left[\log\left(\frac{1}{d} \sum_{i=1}^d \sigma_i^2(0; X)\right)\right] \\
&\le \frac{d}{2} \log\left(\mathbb{E}_{X}\left[\frac{1}{d} \sum_{i=1}^d \sigma_i^2(0; X)\right]\right).
\end{align*}

By Proposition~\ref{prop:rms}, the inner expectation evaluates exactly to the squared global FTSS $g(0)^2$. Thus,
\[
\frac{d}{2} \log\bigl(g(0)^2\bigr) = d \log g(0).
\]
Substituting this bound back into the entropy expression completes the proof.
\end{proof}

The entropy compression bound established in Proposition~\ref{prop:entropy_bound} demonstrates that a drop in the global FTSS $g(0)$ fundamentally restricts the information capacity of the generated distribution.

\subsection{Algorithmic Implementation}
\label{sec:estimation}

Evaluating the state-transition matrix $\Phi(1,t; x_t)$ directly via automatic differentiation requires $d$ backward passes to construct a dense Jacobian matrix, which is computationally prohibitive for high-dimensional models. To bypass this, we directly instantiate Definition~\ref{def:ftss} using a parallelized, gradient-free Monte Carlo estimator based on directional finite differences. As detailed in Algorithm~\ref{alg:gain}, we approximate the global FTSS by injecting isotropic perturbations $\varepsilon U$ into cached intermediate states, mapping them to the target terminal space via forward-pass trajectory simulations, and computing the RMS gain. The estimator involves two approximations: a finite step size $\varepsilon$ for the directional derivative, and a finite number of trajectories $K$ for the expectation over $x_t$. As $\varepsilon \to 0$ and $K \to \infty$, the estimator is consistent and converges to the true FTSS $g(t)$.

\begin{algorithm}[htbp]
\caption{Parallel Gradient-Free Estimation of FTSS}
\label{alg:gain}
\begin{algorithmic}[1]
\Require Velocity field $v_\theta$, initial states $X_0 \in \mathbb{R}^{K \times d}$, grid $\{t_j\}_{j=1}^T$, step $\varepsilon$.
\Ensure Empirical FTSS values $\{\hat{g}(t_j)\}_{j=1}^T$.
\State Compute $X_1 = \phi_{0 \to 1}(X_0)$ and cache intermediate states $\{X_{t_j}\}_{j=1}^T$
\For{$j = 1$ \textbf{to} $T$}
    \For{$k = 1$ \textbf{to} $K$}
        \State Sample $u \sim \mathrm{Unif}(S^{d-1})$
        \State $\delta = \varepsilon \cdot u$
        \State $X_1' = \phi_{t_j \to 1}(X_{t_j}^{(k)} + \delta)$
        \State $r_{k,j} = \|X_1' - X_1^{(k)}\| / \varepsilon$
    \EndFor
    \State $\hat{g}(t_j) = \left( \frac{1}{K} \sum_{k=1}^K r_{k,j}^2 \right)^{1/2}$
\EndFor
\State \Return $\{\hat{g}(t_j)\}_{j=1}^{T}$
\end{algorithmic}
\end{algorithm}

\section{Auditing Memorization via Spectral Sensitivity}
\label{sec:bottleneck_framework}

We analyze how generative velocity fields behave across different training data sizes, model architectures, and datasets. Our main finding is that the Spectral Collapse Ratio $M := g_{\min} / g_{\min}^{\text{full}}$, where $g_{\min} := \min_{t \in [0,1]} g(t)$, characterizes whether a continuous-time model memorizes or generalizes. Here, $g_{\min}^{\text{full}}$ represents the reference baseline minimum of a healthy model trained on the full dataset.

The memorization experiment is designed around data scarcity. Data scarcity drives this geometric change by altering the nature of the learned target distribution. In the low-data regime where the model memorizes a small pool of $N$ training samples, the target distribution turns into a set of isolated, discrete points. To reach these sparse points, the velocity field pinches the high-dimensional flow of paths into a lower dimension. This loss of trajectory volume causes the FTSS to drop at the bottleneck, making $g_{\min}$ a direct indicator of memorization.

\subsection{Finite-Time Sensitivity and Bottleneck Dynamics}
\label{sec:bottleneck_dynamics}

To see how $g(t)$ behaves in real models, we plot FTSS curves across different training data sizes, model architectures, and datasets (see Figure~\ref{fig:crossmethod} in Appendix~\ref{app:implementation}). We observe two distinct phases during the trajectory:
\begin{itemize}
    \item \textbf{Early Stage ($t \to 0$):} The early stage varies across architectures and datasets. The FTSS can go above 1 ($g(t) > 1$) because the model pushes paths outward to match the scale of the target data.
    \item \textbf{Late Stage ($t \to 1$):} The paths shrink ($g(t) < 1$) as they come close to their target.
\end{itemize}

Because the initialization phase is heavily influenced by the global scale of the data, utilizing the bottleneck minimum $g_{\min}$ isolates the intrinsic compression occurring within the flow. Geometrically, this trajectory pinching represents a reduction in the effective dimensionality of the system's tangent space. We can formally characterize this structural degeneration by introducing a continuous relaxation of matrix rank.

\begin{remark}[Stable Rank and Memorization]\label{rm:srank}
The stable rank $\mathrm{srank}(W) := \|W\|_F^2 / \|W\|_2^2$ provides a continuous relaxation of the matrix rank \cite{vershynin2012introduction, bartlett2017spectrally}. As established in Proposition~\ref{prop:rms}, the squared FTSS $g(t)^2$ captures the expected Frobenius norm of the state-transition matrix. Given that its spectral norm is $\sigma_{\max}(t)^2$, the effective dimensionality along the generation path follows $\mathrm{srank}(t) \propto g(t)^2 / \sigma_{\max}(t)^2$.

Because both the generalizing and memorizing models map the identical isotropic prior to the same macroscopic target domain, the global transport distance is conserved. This geometric constraint forces the primary transport scale to remain stable ($\sigma_{\max} \approx \sigma_{\max}^{\text{full}}$). Under this approximation, the spectral norm cancels, and the squared Spectral Collapse Ratio quantifies the proportional rank collapse at the temporal bottleneck:
\begin{equation}
M^2 = \left( \frac{g_{\min}}{g_{\min}^{\text{full}}} \right)^2 \approx \frac{\min_t \mathrm{srank}(t)}{\min_t \mathrm{srank}^{\text{full}}(t)}.
\end{equation}
\end{remark}

\subsection{The Spectral Sensitivity Auditing Framework}
\label{sec:diagnostic}
We formalize this process in Algorithm~\ref{alg:memorization} as an internal geometric audit. The Spectral Collapse Ratio curves for different numbers of training samples are shown in Fig.~\ref{fig:score} in Appendix~\ref{app:implementation}. The results show that the ratio increases monotonically with the number of samples until reaching a threshold. Crucially, the completely non-overlapping variance intervals between extreme data scarcity and the fully resolved data regime ($M \approx 1.0$ at $N \ge 1000$) establish a strict statistical separability. This stark empirical divide validates the ability to detect memorization as an intrinsic structural property, concrete-proofing the theoretical framework by mapping the geometric collapse directly within the generative flow.

\begin{algorithm}[htbp]
\caption{Geometric Memorization Audit via Spectral Sensitivity}
\label{alg:memorization}
\begin{algorithmic}[1]
\Require Model field $v_\theta$, reference baseline $g_{\min}^{\text{full}}$, decision threshold $\tau \in (0,1)$.
\Ensure Audit decision state $\mathcal{H} \in \{\mathcal{H}_0, \mathcal{H}_1\}$.
\State Evaluate temporal trajectory FTSS values $\{\hat{g}(t_j)\}_{j=1}^T$ via Algorithm~\ref{alg:gain}
\State $g_{\min} \gets \min_j \hat{g}(t_j)$
\State $M \gets \frac{g_{\min}}{g_{\min}^{\text{full}}}$
\If{$M < \tau$}
    \State \Return $\mathcal{H}_1$ \Comment{Memorization detected: memorized model}
\Else
    \State \Return $\mathcal{H}_0$ \Comment{Manifold conservation maintained: generalizing model}
\EndIf
\end{algorithmic}
\end{algorithm}

\section{Discussion and Conclusions}
\label{sec:discussion_conclusion}
In this paper, we introduced the FTSS $g(t)$, a scalable, gradient-free geometric metric that measures the average directional sensitivity and volume deformation of continuous-time generative flows. We established that low-data regimes trigger a severe spectral collapse, manifesting as a sharp drop in the bottleneck minimum $g_{\min}$. Utilizing this signature, we developed an internal auditing framework based on the empirical Spectral Collapse Ratio $M = g_{\min} / g_{\min}^{\text{full}}$. This framework successfully detects overfitted continuous-time models directly from their internal flow dynamics, completely eliminating the need for post-hoc data comparisons or training pool access.

Our findings demonstrate that generative memorization is fundamentally an \emph{intrinsic geometric pathology} rather than a purely statistical phenomenon. When constrained by data scarcity, continuous-time velocity fields compress trajectory volume by pinching high-dimensional flows toward isolated, discrete target points. Crucially, because the FTSS $g(t)$ tracks the RMS singular value of the forward state-transition matrix, it exposes this structural contraction as a severe rank collapse at a localized temporal bottleneck ($g_{\min}$). The monotonic relationship between the Spectral Collapse Ratio $M$ and the training sample size $N$ confirms that this dimensional collapse is an invariant signature of data scarcity. From a practical perspective, our parallelized, gradient-free Monte Carlo estimator circumvents the prohibitive $O(d)$ computational bottleneck associated with explicit spatial Jacobian extraction, shifting the operational cost to efficient forward-pass finite differences.

\paragraph{Code availability.}
Code and experimental configurations are publicly available at:
\url{https://github.com/ShuchanWang/ftss-memorization}.

% Print the bibliography
\printbibliography

\appendix

\section{Detailed Experimental Results}
\label{app:implementation}

This appendix provides the complete empirical characterization of our geometric auditing framework. Specifically, Figure~\ref{fig:crossmethod} tracks the full temporal trajectories of the empirical FTSS $g(t)$ across varying training set sizes $N$, mapping how severe data scarcity forces an internal dimensional bottleneck during generation. To aggregate these dynamics into a practical diagnostic tool, Figure~\ref{fig:score} displays the scalar Spectral Collapse Ratio $M$ as a function of $N$, illustrating the clear phase transition from rank-collapsed memorization to manifold-preserving generalization across all tested architectures and datasets.

\begin{figure}[htbp]
    \centering
    \includegraphics[width=0.47\textwidth]{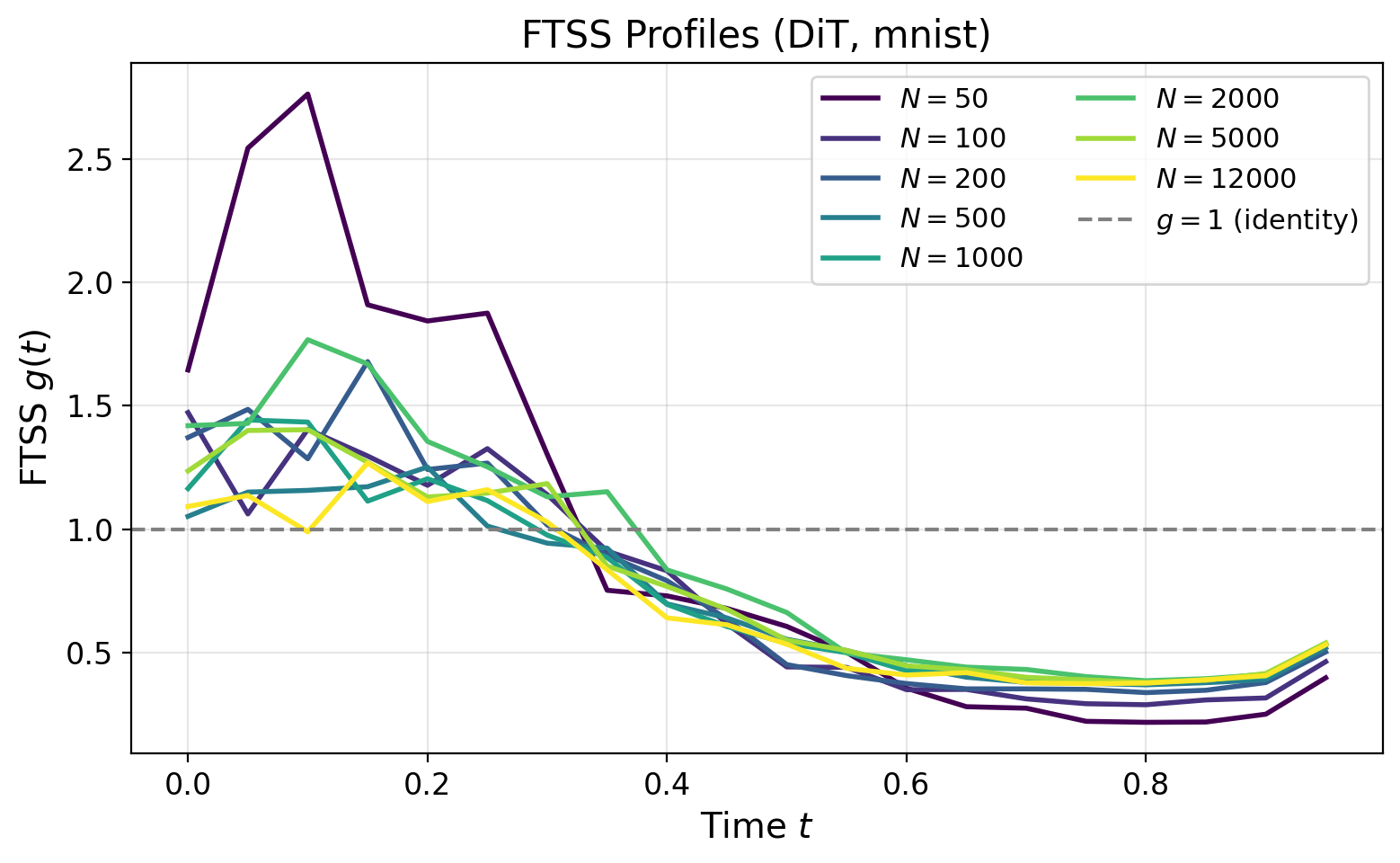}
    \hfill
    \includegraphics[width=0.47\textwidth]{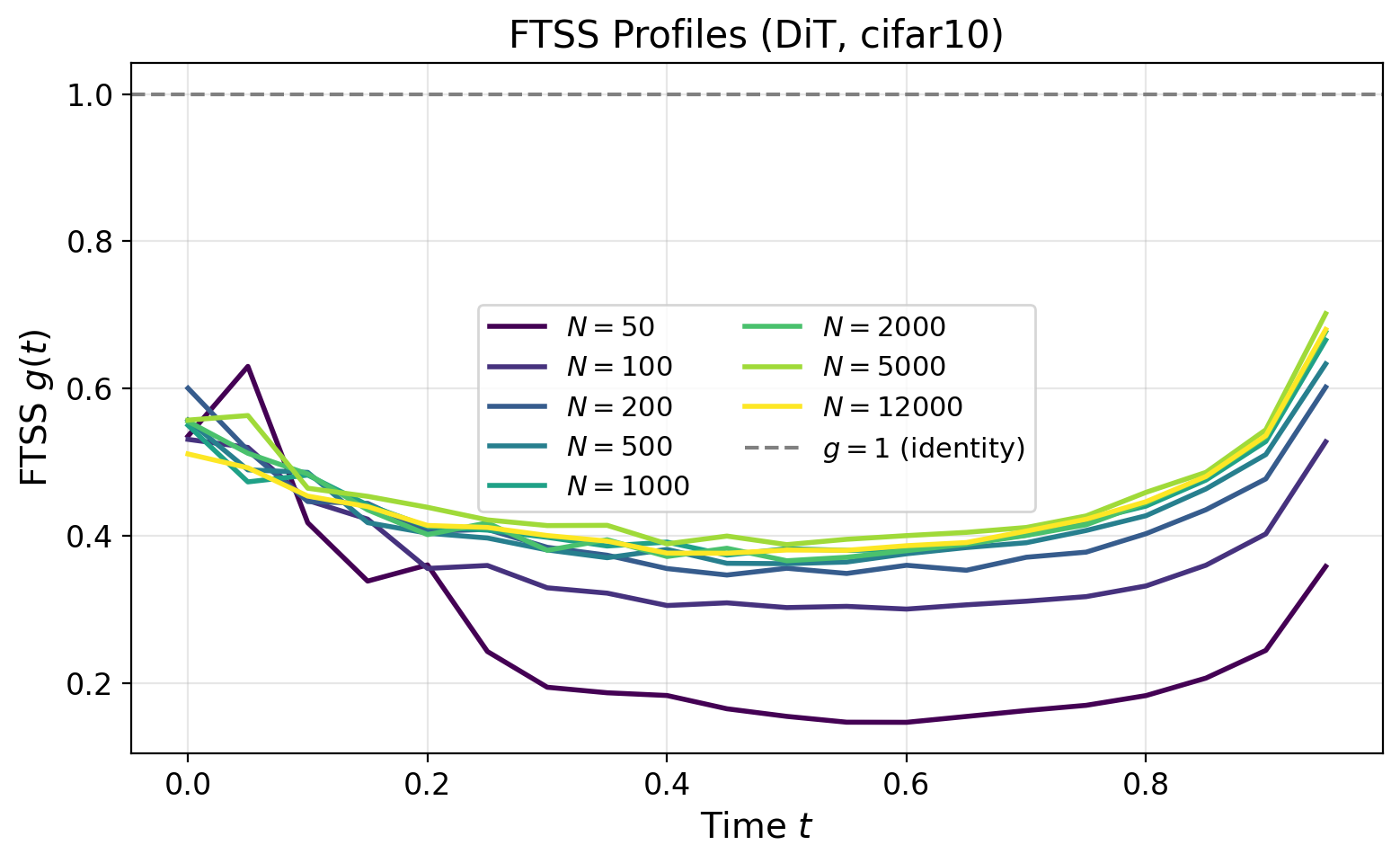}
    \medskip

    \includegraphics[width=0.47\textwidth]{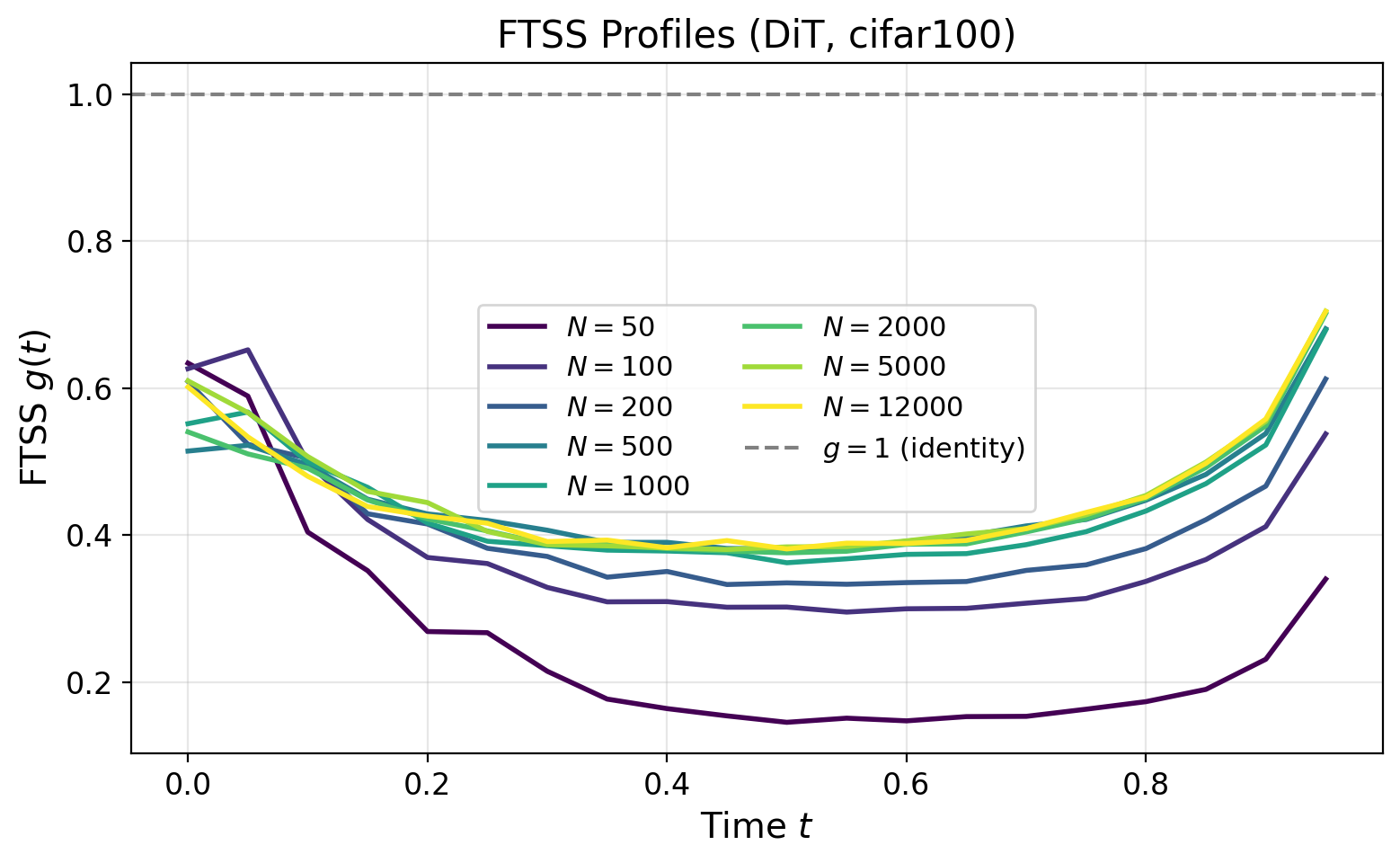}
    \hfill
    \includegraphics[width=0.47\textwidth]{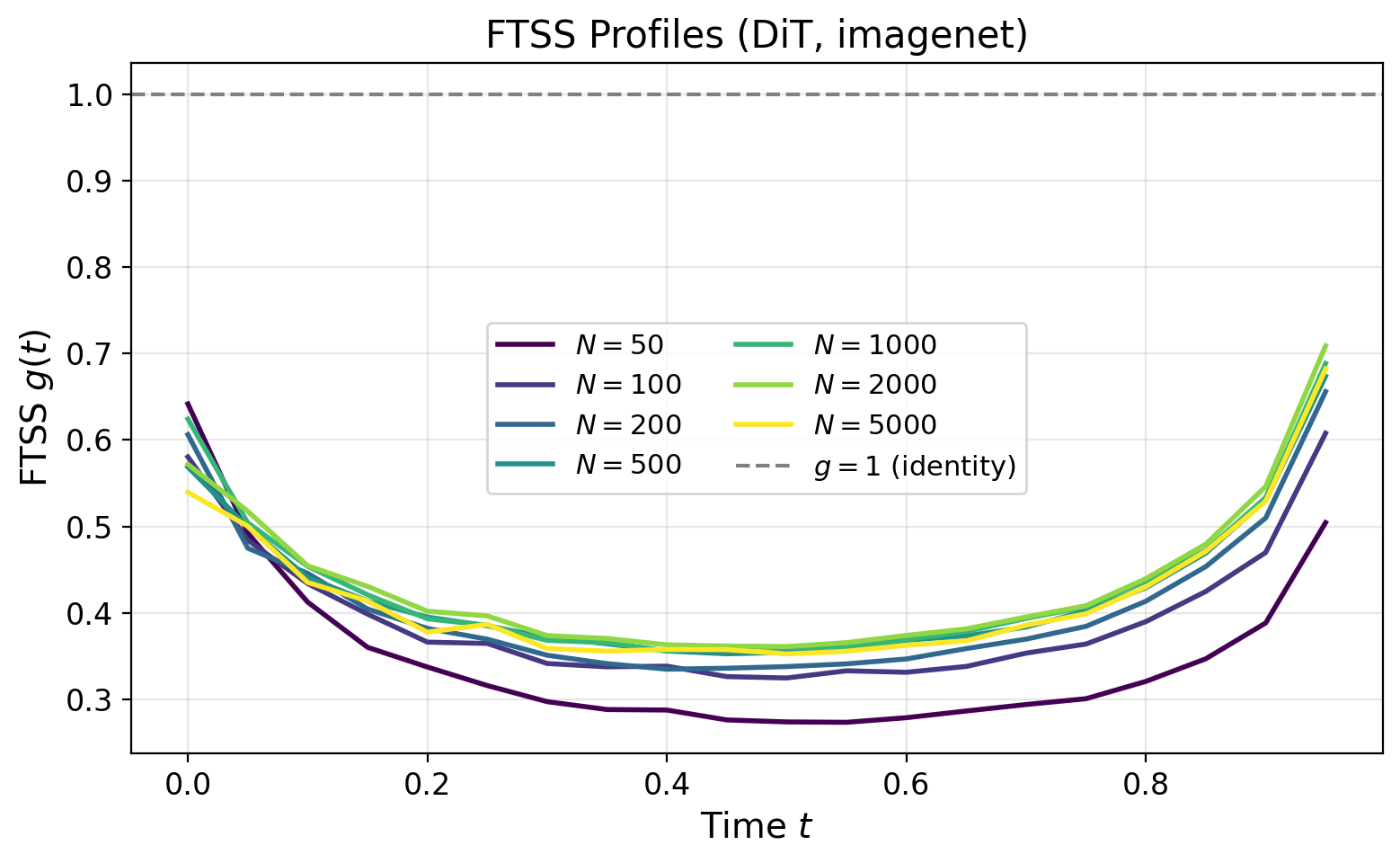}
    \medskip

    \includegraphics[width=0.47\textwidth]{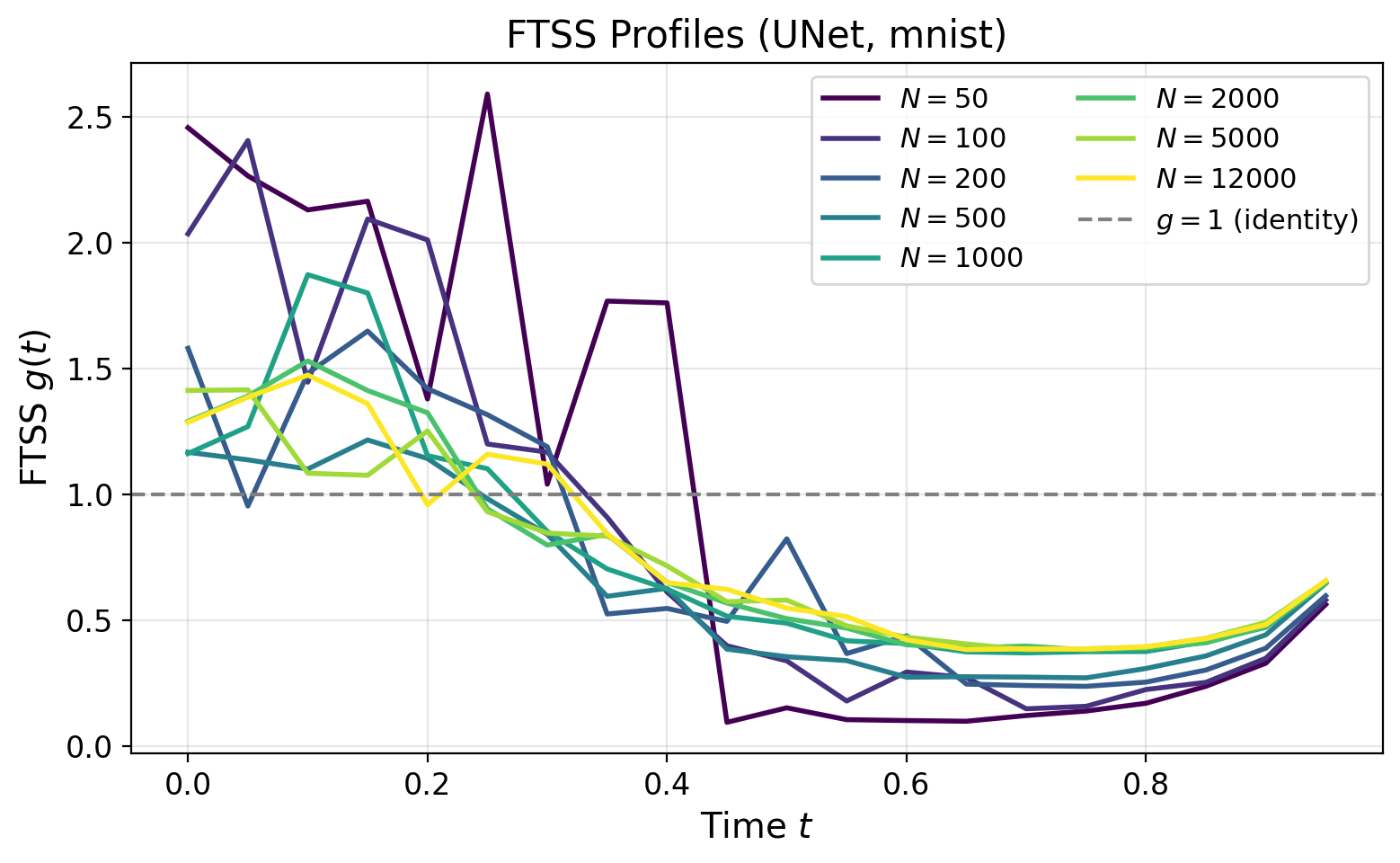}
    \hfill
    \includegraphics[width=0.47\textwidth]{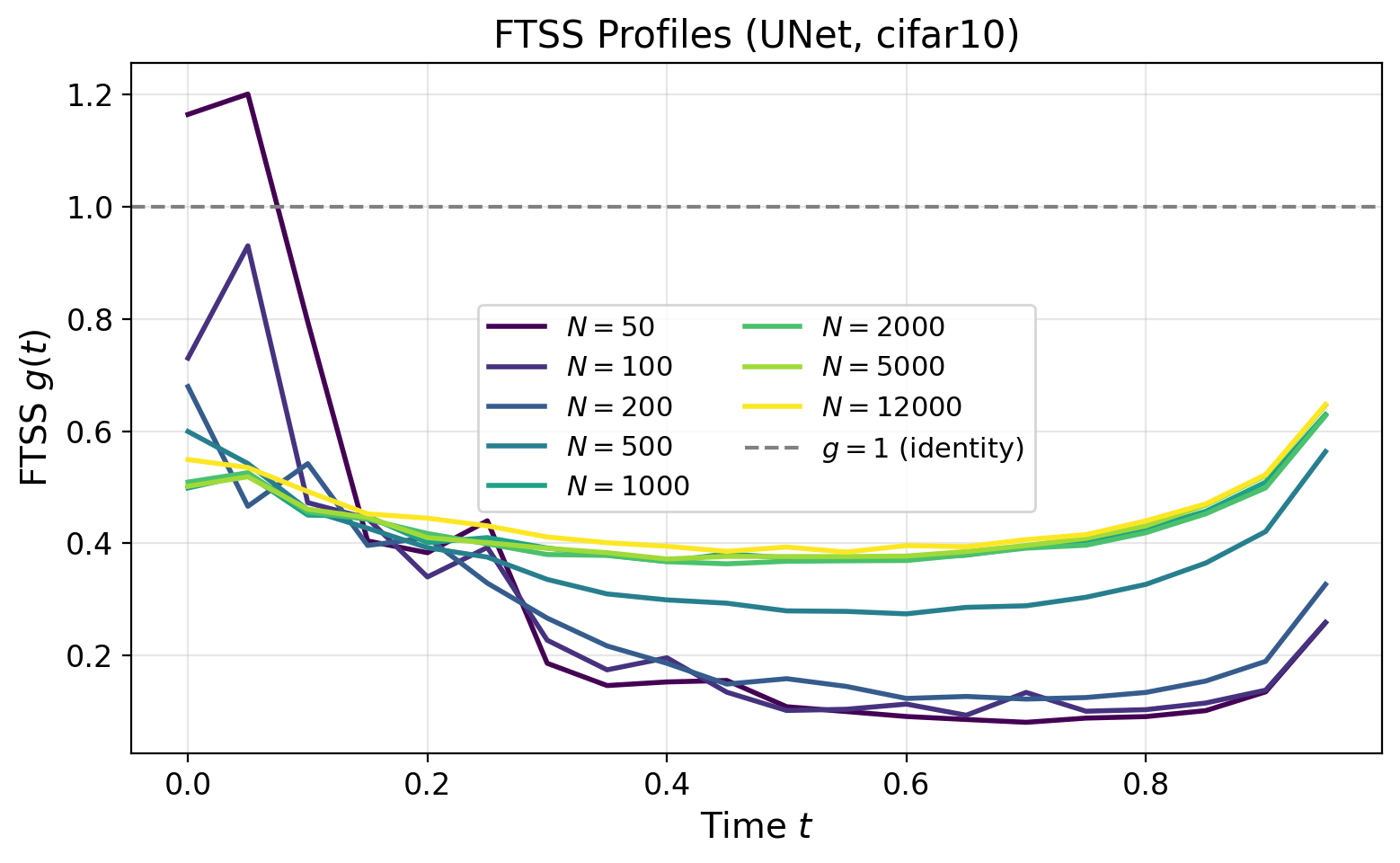}
    \medskip

    \includegraphics[width=0.47\textwidth]{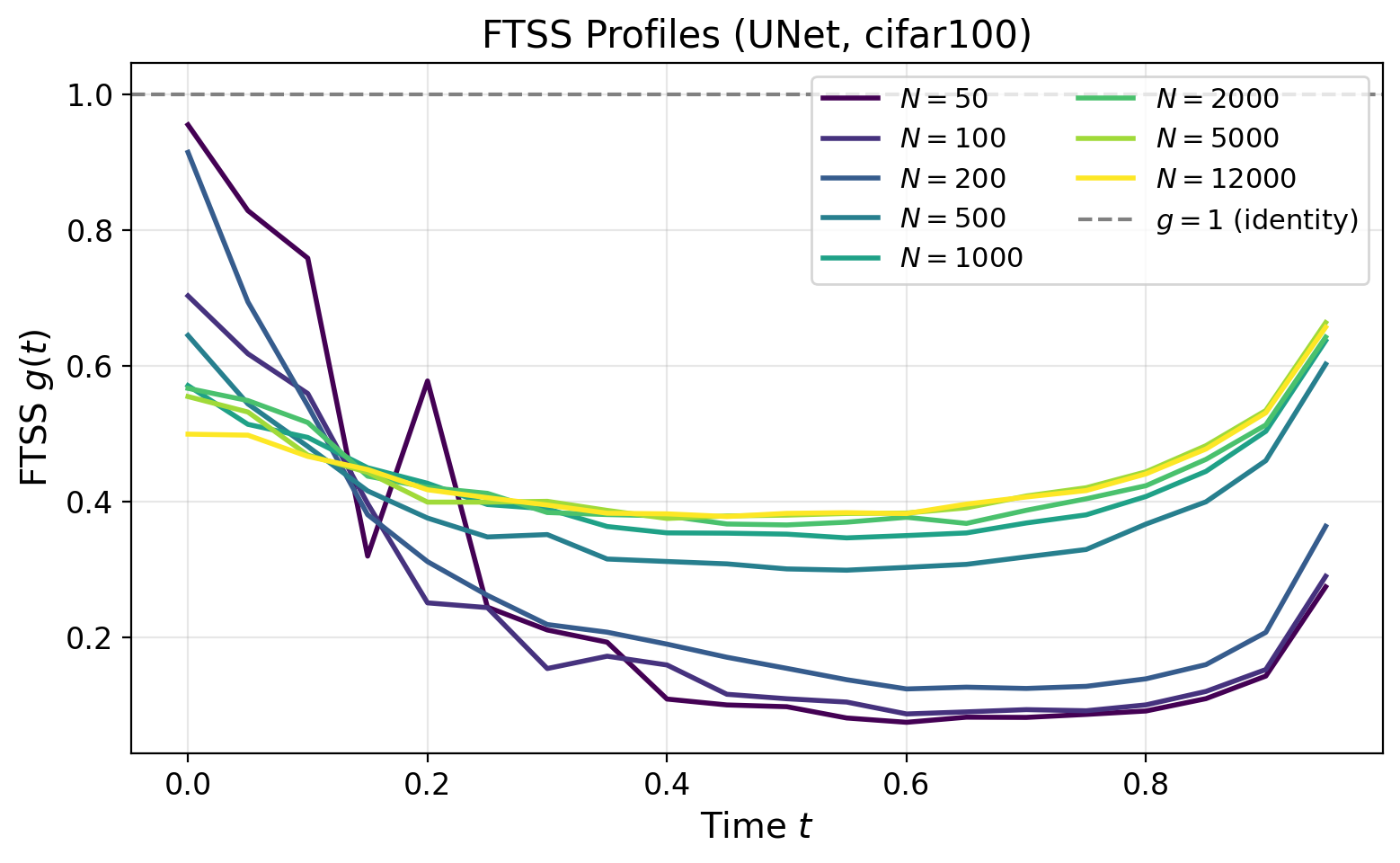}
    \hfill
    \includegraphics[width=0.47\textwidth]{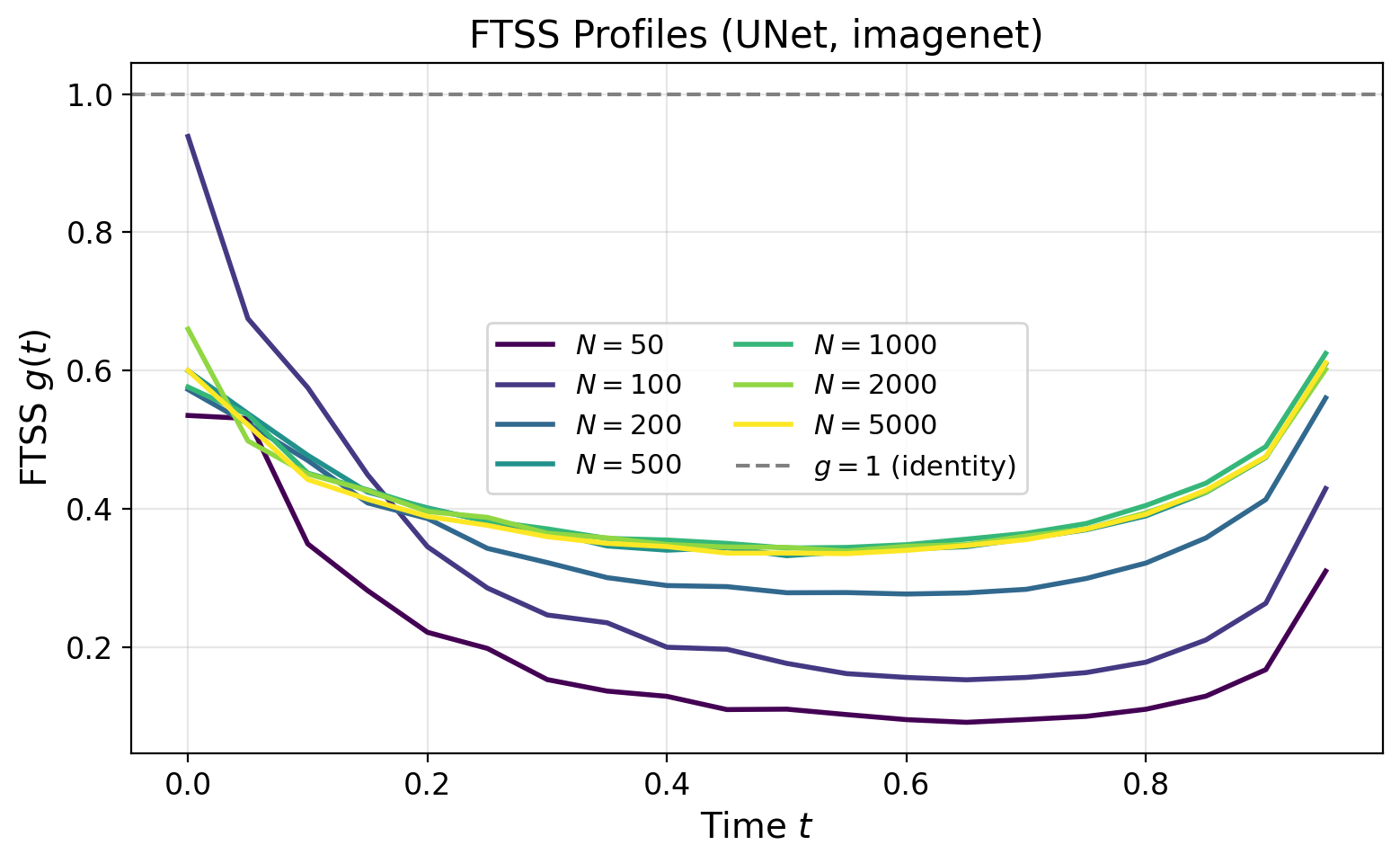}
    \medskip

    \includegraphics[width=0.47\textwidth]{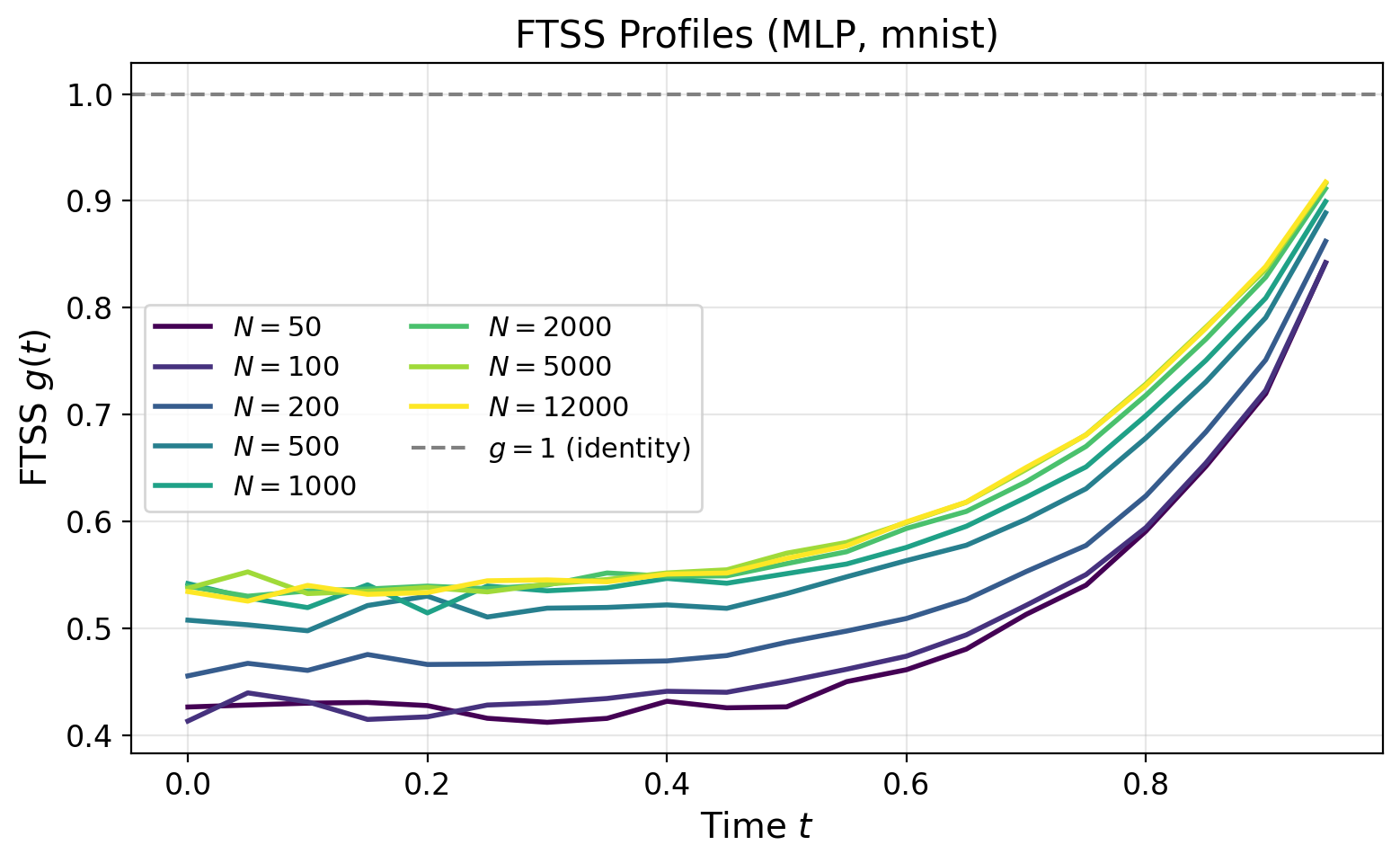}
    \hfill
    \begin{minipage}[b]{0.47\textwidth}\quad\end{minipage} % Invisible spacer to perfectly balance the 9th image
    
    \caption{FTSS curves across model architectures and datasets.}
    \label{fig:crossmethod}
\end{figure}

\begin{figure}[htbp]
    \centering
    \includegraphics[width=0.41\textwidth]{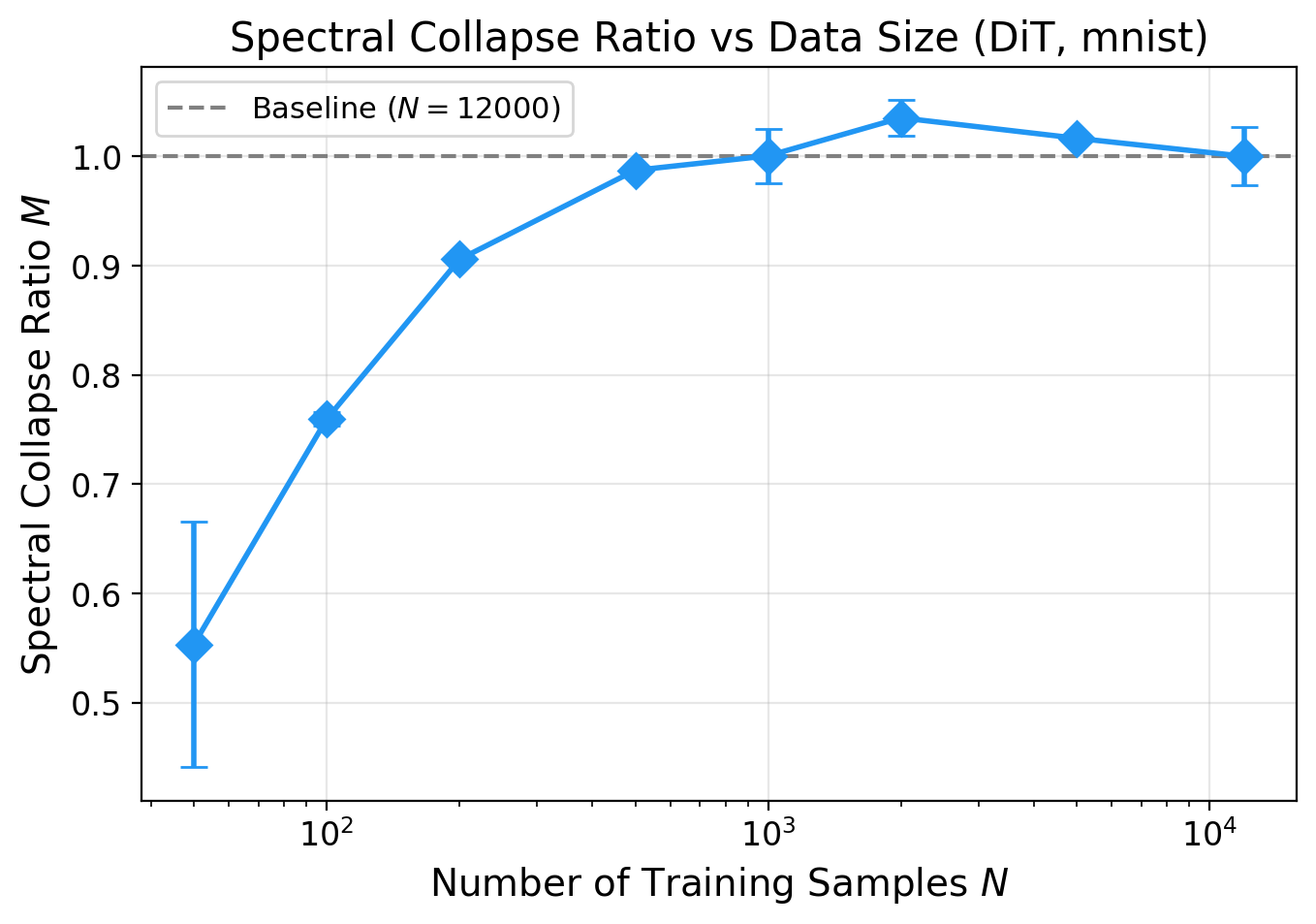}
    \hfill
    \includegraphics[width=0.41\textwidth]{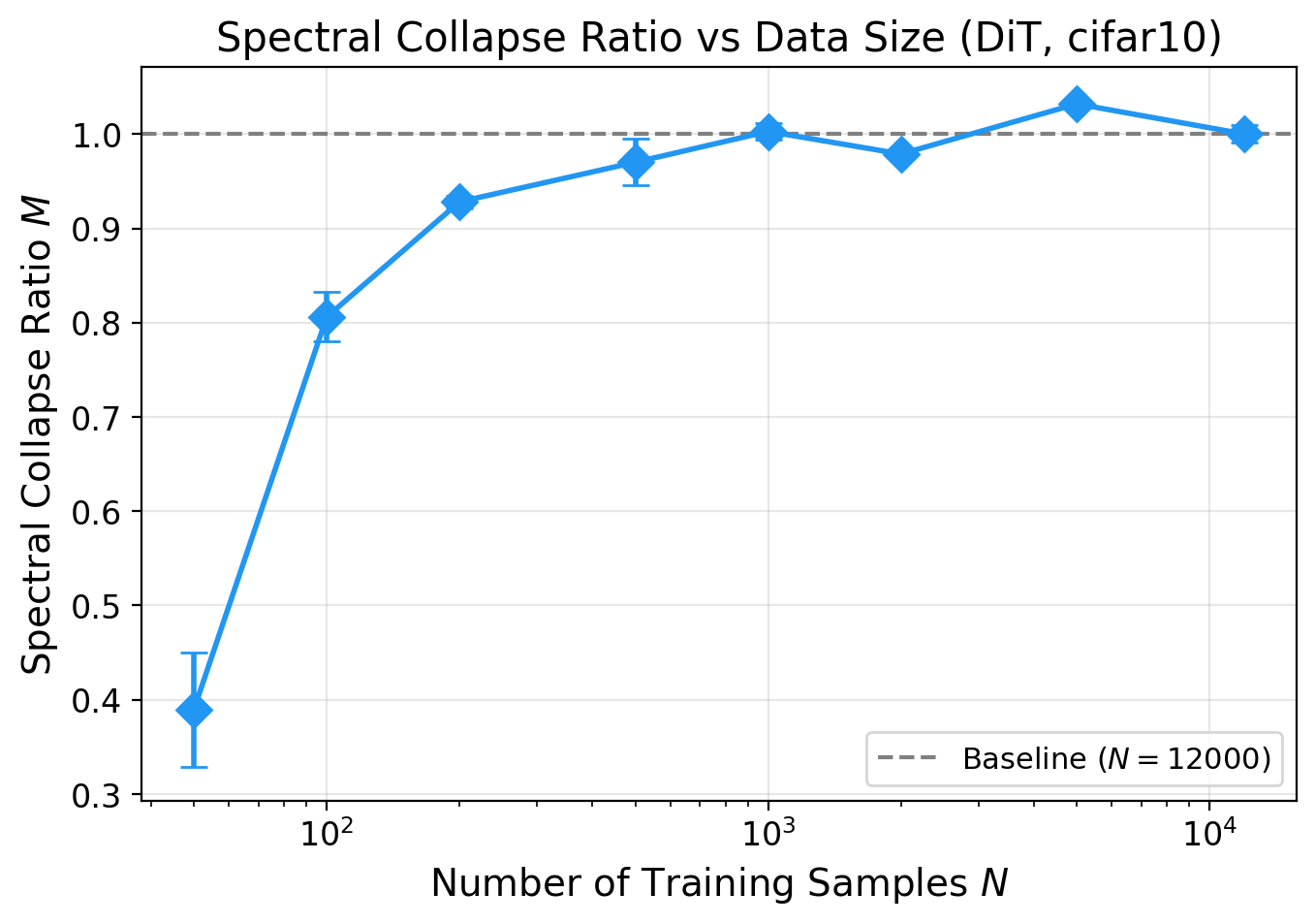}
    \medskip

    \includegraphics[width=0.41\textwidth]{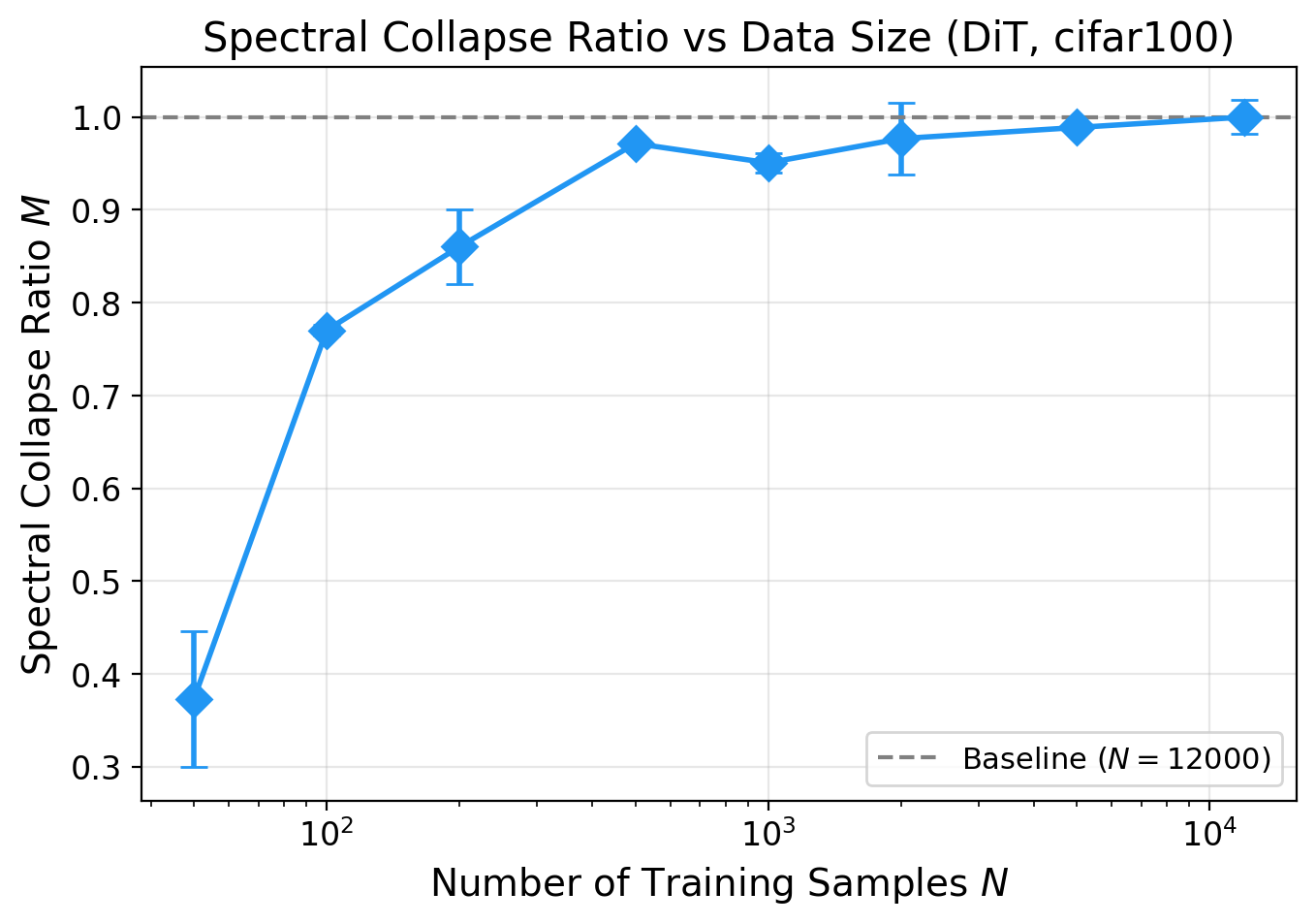}
    \hfill
    \includegraphics[width=0.41\textwidth]{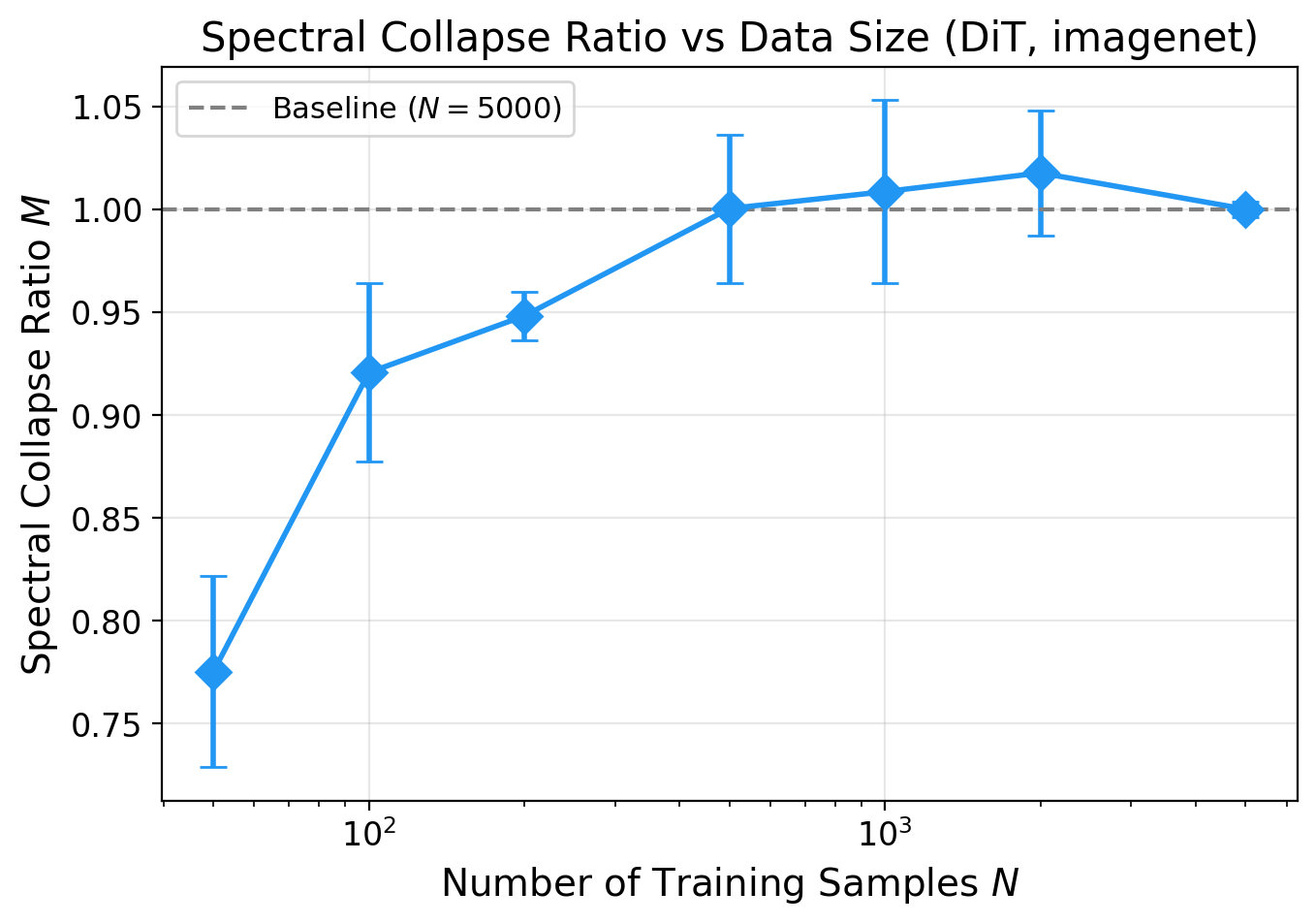}
    \medskip

    \includegraphics[width=0.41\textwidth]{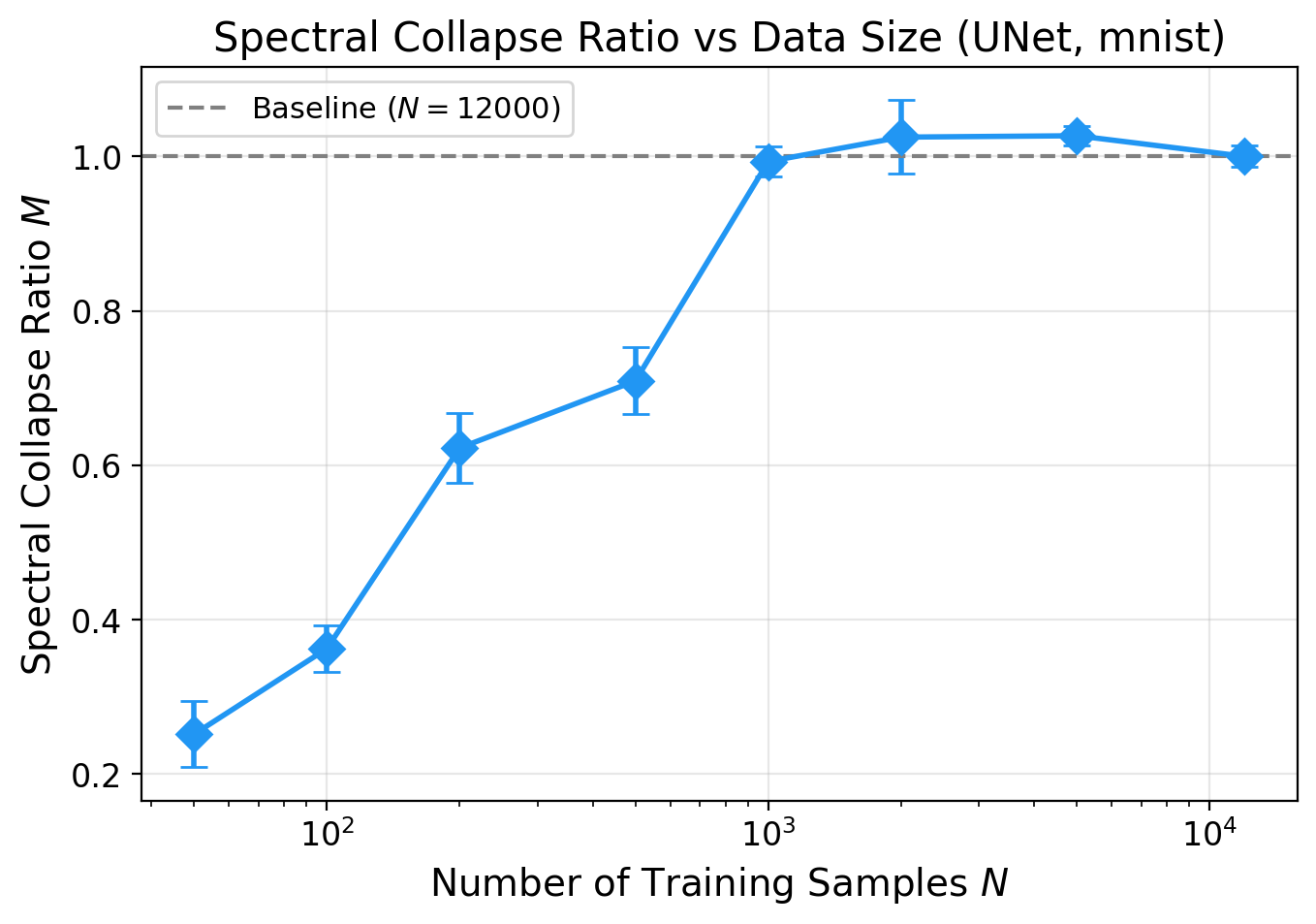}
    \hfill
    \includegraphics[width=0.41\textwidth]{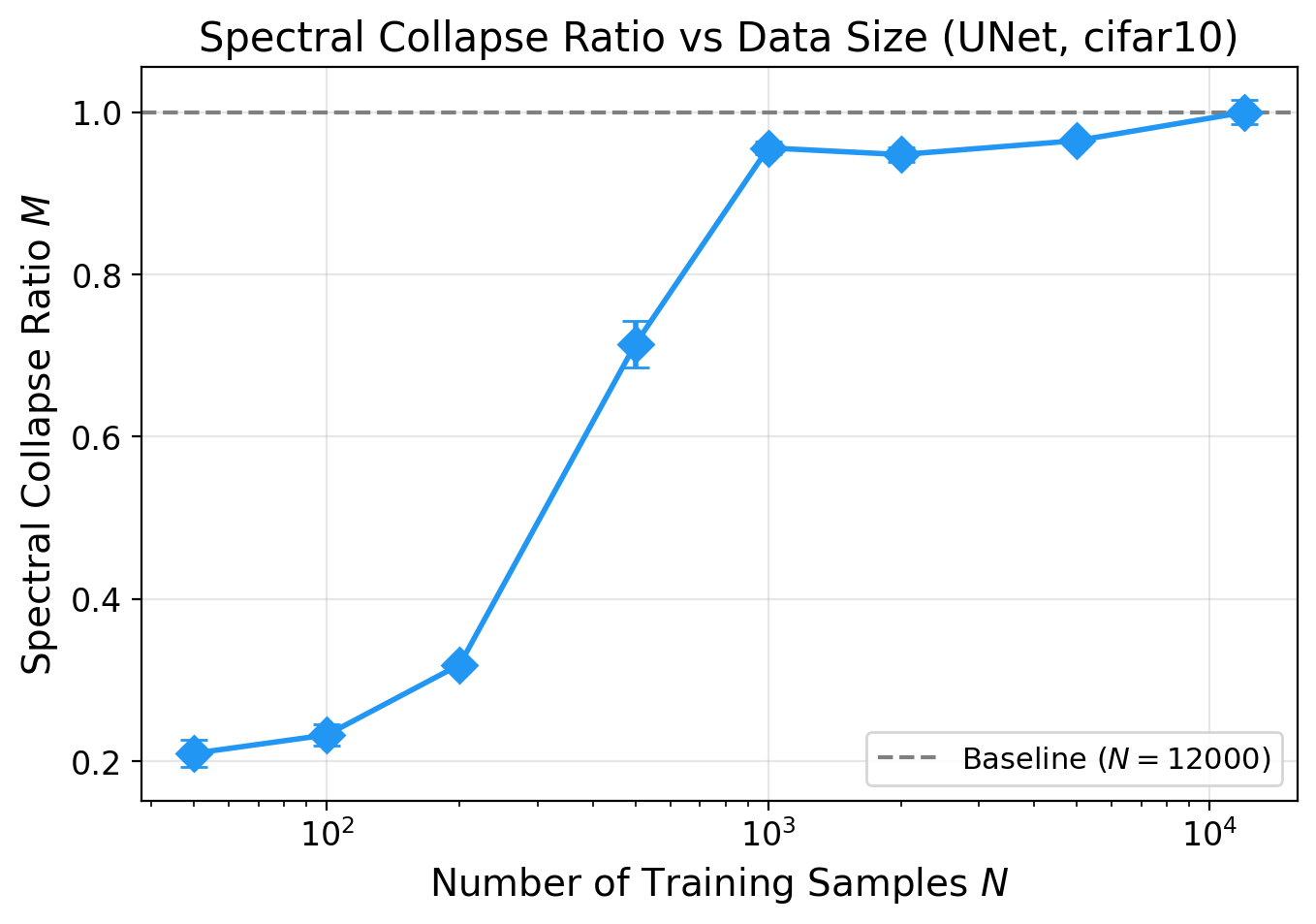}
    \medskip

    \includegraphics[width=0.41\textwidth]{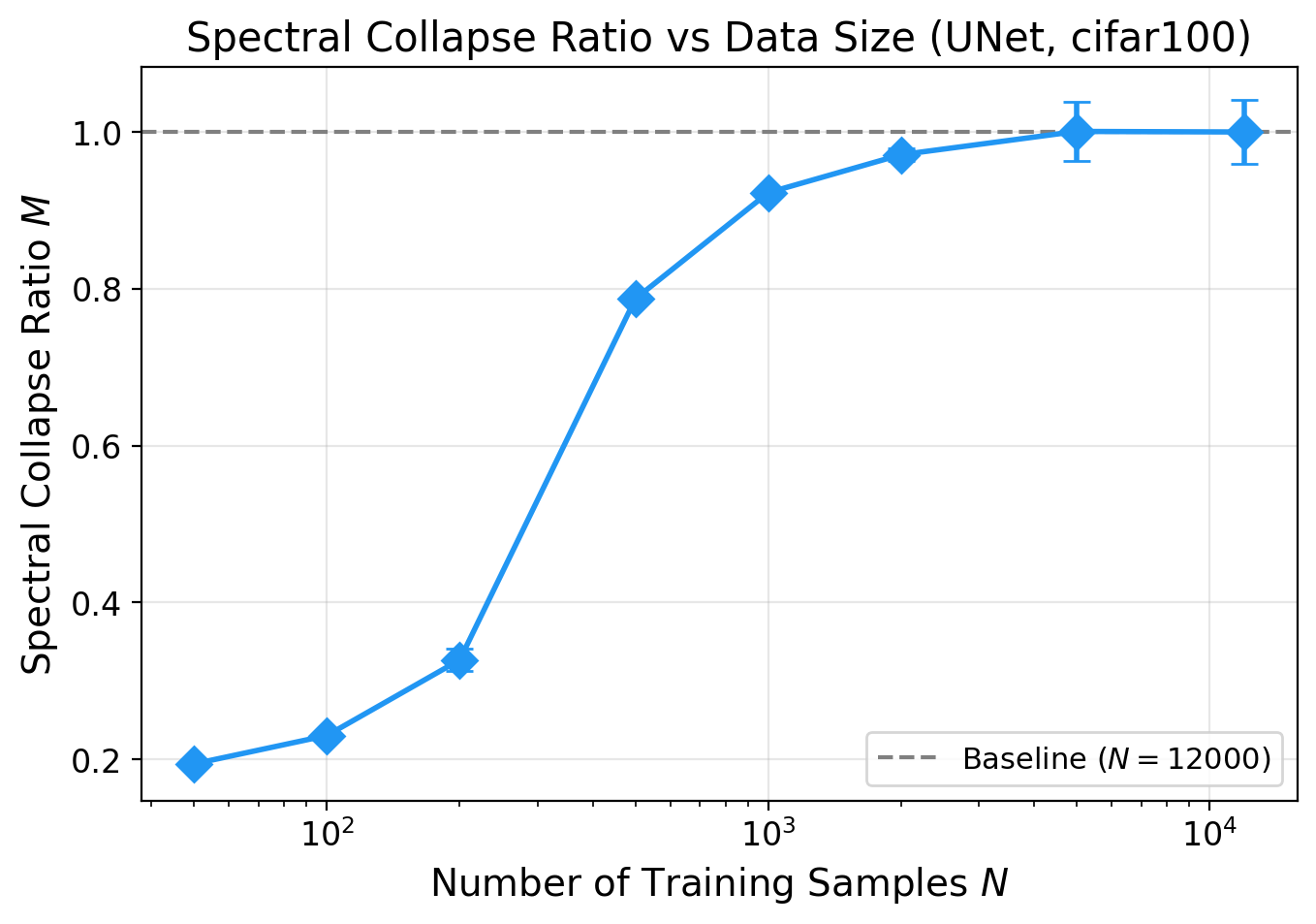}
    \hfill
    \includegraphics[width=0.41\textwidth]{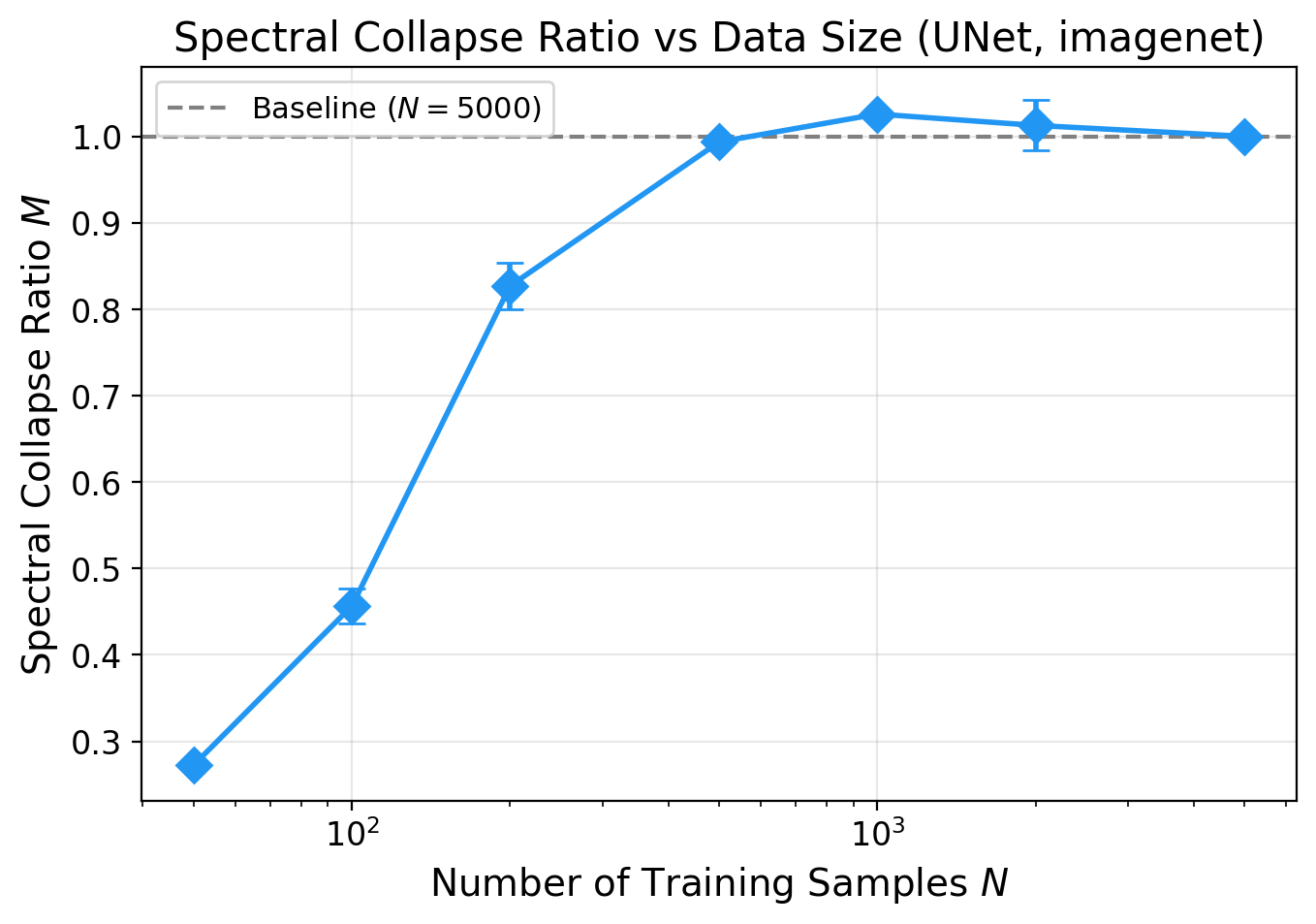}
    \medskip

    \includegraphics[width=0.41\textwidth]{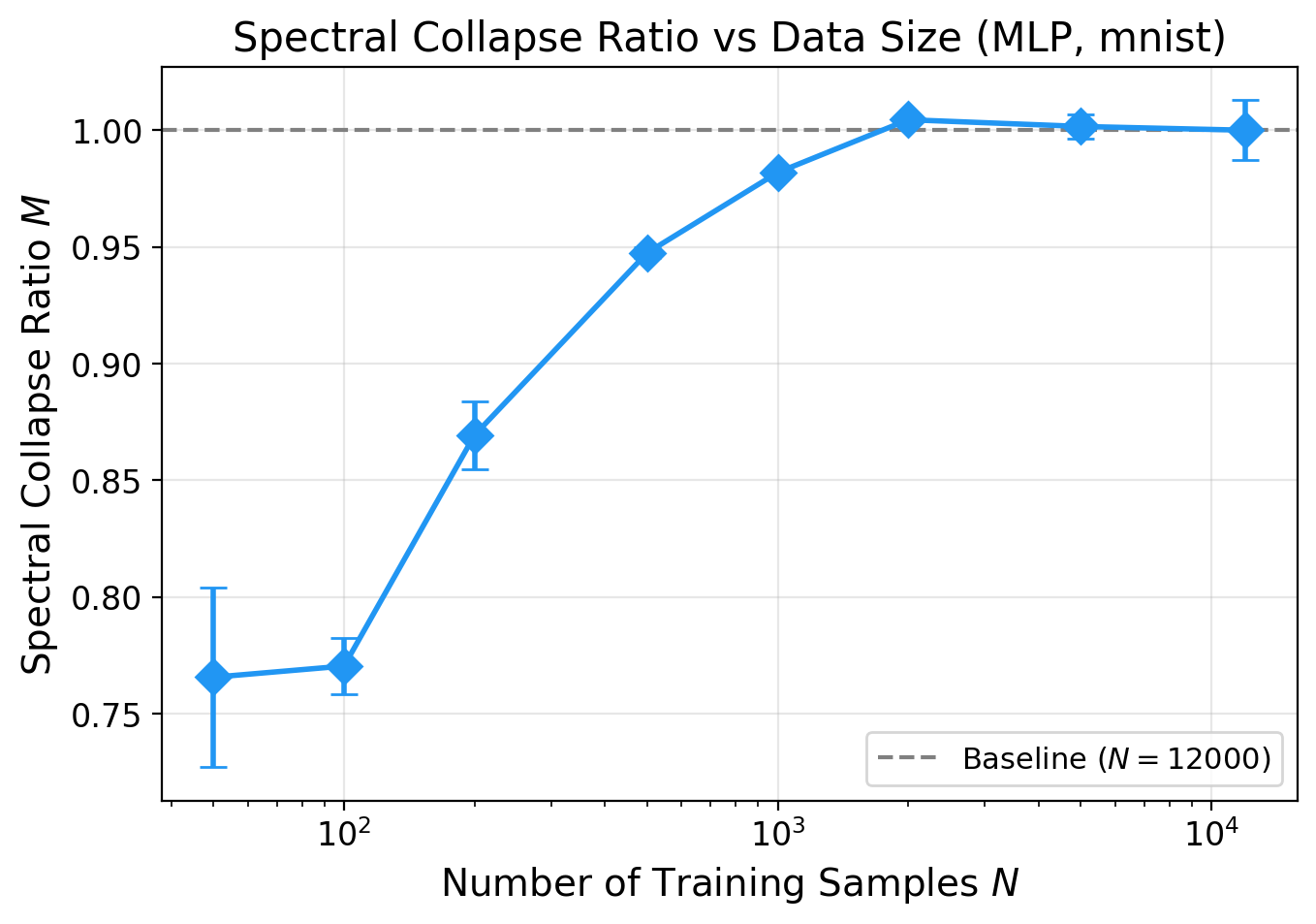}
    \hfill
    \begin{minipage}[b]{0.41\textwidth}\quad\end{minipage} % Invisible spacer to perfectly balance the 9th image
    
    \caption{Spectral Collapse Ratio vs. number of training samples.}
    \label{fig:score}
\end{figure}

We describe the complete experimental configuration for reproducibility. All experiments use flow matching with straight-line conditional probability paths.

\subsection{Datasets}

We train on four image datasets listed in Table~\ref{tab:datasets}.

\begin{table}[htbp]
\centering
\caption{Datasets used for training and evaluation.}
\label{tab:datasets}
\begin{tabular}{@{}lcccc@{}}
\toprule
Dataset & Resolution & Channels & Dimension $d$ & Training samples \\
\midrule
MNIST & $28\times28$ & 1 & 784 & 60,000 \\
CIFAR-10 & $32\times32$ & 3 & 3,072 & 50,000 \\
CIFAR-100 & $32\times32$ & 3 & 3,072 & 50,000 \\
ImageNet (imagenette) & $64\times64$ & 3 & 12,288 & 9,469 \\
\bottomrule
\end{tabular}
\end{table}

For ImageNet-scale experiments, we use the imagenette subset (10-class) due to computational constraints. All images are normalized to $[-1, 1]$ via $\texttt{Normalize}((0.5,\ldots),(0.5,\ldots))$.

\subsection{Data Size Sweep}

To study memorization, we train models on subsets of varying size from each dataset. For MNIST, CIFAR-10, and CIFAR-100, we sweep over
\begin{equation}
N \in \{50, 100, 200, 500, 1000, 2000, 5000, 12000\},
\end{equation}
where the largest configuration ($N = 12,\!000$) serves as our reference baseline. For imagenette, we sweep over
\begin{equation}
N \in \{50, 100, 200, 500, 1000, 2000, 5000\},
\end{equation}
with the largest subset capped at $5,\!000$ samples (the full imagenette dataset contains $9,\!469$ samples). At each data size, two independent training runs are performed with different random seeds ($\text{seed} = 42 + 1000 \times \text{run\_id}$) to estimate variance. The baseline minimum $g_{\min}^{\text{full}}$ used in the Spectral Collapse Ratio is computed from the model trained on the largest subset for each dataset (12,000 for MNIST/CIFAR-10/CIFAR-100; 5,000 for imagenette).

\subsection{Architectures}

We evaluate three model families spanning different inductive biases:

\paragraph{UNet.} A convolutional UNet with $4$ downsampling/upsampling blocks, 
base channels $64$, instance normalization, and SiLU activations. Time 
conditioning uses a 2-layer MLP projecting the scalar time to a $256$-dimensional 
embedding; the embedding is added to the features at the lowest spatial resolution.

\paragraph{Diffusion Transformer (DiT).} A patch-based 
transformer~\cite{peebles2023scalable} with patch size $4$. Images are split 
into non-overlapping patches, embedded to dimension $256$, and processed by 
$4$ transformer blocks with $4$ attention heads and MLP expansion factor $4$. 
Time conditioning uses adaptive layer normalization (adaLN). Positional 
information is provided by learned positional embeddings. Dropout rate is 
$0.1$. The same architecture is used across all datasets.

\paragraph{MLP.} A fully-connected network with $5$ hidden layers of width 
$2048$, LayerNorm, SiLU activations, and Dropout($0.1$). Time is concatenated 
to the flattened input. MLP experiments are restricted to MNIST; the 
architecture does not scale effectively to higher-dimensional RGB images due 
to the absence of spatial inductive biases.

\begin{table}[htbp]
\centering
\caption{Model configurations and approximate parameter counts.}
\label{tab:architectures}
\begin{tabular}{@{}lccc@{}}
\toprule
Architecture & Configuration & Parameters (MNIST) & Parameters (CIFAR/ImageNet) \\
\midrule
UNet & $4$ blocks, base\_ch$=64$ & $\sim$1.5M & $\sim$2.0M \\
DiT & $L=4$, $d=256$, $h=4$ & $\sim$3.5M & $\sim$5.6M \\
MLP & $5$ layers, width$=2048$ & $\sim$16.8M & -- \\
\bottomrule
\end{tabular}
\end{table}

\subsection{Training}

All models are trained to regress the target velocity field $u(x_t, t) = x_1 - x_0$ under the conditional flow matching objective~\cite{lipmanflow}:
\begin{equation}
\mathcal{L}(\theta) = \mathbb{E}_{t, x_0, x_1}\left[\|v_\theta(x_t, t) - (x_1 - x_0)\|^2\right],
\end{equation}
where $t \sim \mathcal{U}(0,1)$, $x_0 \sim \mathcal{N}(0, I)$, $x_1$ is a training sample, and $x_t = (1-t)x_0 + t x_1$.

Training hyperparameters are summarized in Table~\ref{tab:training}. 
For the MLP, we use mixed-precision training (FP16) via PyTorch automatic 
mixed precision (AMP), and maintain an exponential moving average (EMA) of 
model weights with decay $0.999$; the EMA parameters are used for all 
evaluations. The UNet and DiT are trained in standard FP32 without EMA.

\begin{table}[htbp]
\centering
\caption{Training hyperparameters.}
\label{tab:training}
\begin{tabular}{@{}lccc@{}}
\toprule
Hyperparameter & UNet & DiT & MLP \\
\midrule
Training steps & 8,000 & 8,000 & 20,000 \\
Batch size & 64 & 64 & 64 \\
Learning rate & $5\times10^{-4}$ & $5\times10^{-4}$ & $1\times10^{-4}$\\
Optimizer & Adam & AdamW & AdamW \\
Weight decay & 0 & 0.01 & 0.01 \\
LR schedule & Cosine & Cosine & OneCycle (10\% warmup) \\
Gradient clip & 1.0 & 1.0 & 5.0 \\
EMA & No & No & Yes (0.999) \\
Mixed precision & No & No & Yes \\
\bottomrule
\end{tabular}
\end{table}

\subsection{Finite-Time Spectral Sensitivity Estimation}

The FTSS $g(t)$ is estimated via Algorithm~\ref{alg:gain} in the main text. We use $K = 15$ trajectories, $T = 20$ uniformly-spaced time points in $[0, 0.95]$, and perturbation scale $\varepsilon = 0.1 \times \|x_t\|$. Integration uses Euler's method with $N_{\text{steps}} = 200$ steps. Each perturbation direction $u \sim \text{Unif}(S^{d-1})$ is drawn independently. Velocity predictions are clamped to $[-10, 10]$ and state values to $[-5, 5]$ during integration for numerical stability. The results of the Spectral Collapse Ratio are reported as mean $\pm$ standard deviation over the $2$ training runs.

\section{Relation to Finite-Time Lyapunov Exponents}
\label{app:lyapunov}

The FTSS $g(t)$ and finite-time Lyapunov exponents (FTLEs) both characterize 
the singular-value spectrum of the state-transition matrix $\Phi(1,t)$, but 
they measure fundamentally different aspects of its geometry. This appendix 
clarifies the mathematical distinction.

\subsection{Worst-Case Instability vs.\ Average Sensitivity}

For a fixed initial condition $x_t$ and perturbation direction $u$, the 
FTLE over the interval $[t,1]$ is~\cite{shadden2005definition}:
\begin{equation}
\lambda(t, u; x_t) = \frac{1}{1-t}\log\frac{\|\Phi(1,t;x_t)\,u\|}{\|u\|}.
\end{equation}
It measures the exponential separation rate for a single trajectory $x_t$ and 
a single direction $u$. Maximizing over directions yields the dominant FTLE 
for that trajectory:
\begin{equation}
\lambda_{\max}(t; x_t) = \frac{1}{1-t}\log\sigma_{\max}\bigl(\Phi(1,t;x_t)\bigr),
\end{equation}
where $\sigma_{\max}$ is the largest singular value of the state-transition 
matrix. The dominant FTLE is widely used in chaotic dynamics and turbulent 
mixing to identify the most unstable direction at each point in state 
space~\cite{haller2015lagrangian}.

Averaging $\lambda(t,u;x_t)$ over isotropic directions $u \sim \mathrm{Unif}(S^{d-1})$ 
and trajectories $x_t$ gives a mean FTLE:
\begin{equation}
\bar{\lambda}(t) = \mathbb{E}_{x_t}\!\left[\mathbb{E}_{u}\bigl[\lambda(t,u;x_t)\bigr]\right]
= \frac{1}{d}\;\mathbb{E}_{x_t}\!\left[\sum_{i=1}^{d} \frac{\log\sigma_i(t;x_t)}{1-t}\right].
\end{equation}

In contrast, the FTSS from Proposition~\ref{prop:rms} aggregates differently:
\begin{equation}
g(t)^2 = \frac{1}{d}\;\mathbb{E}_{x_t}\!\left[\sum_{i=1}^{d} \sigma_i^2(t;x_t)\right].
\end{equation}
The key structural difference is the order of operations: the mean FTLE takes 
the logarithm before averaging over directions and trajectories, while the 
FTSS averages the squared singular values first. Consequently, the log and 
expectation do not commute:
\begin{equation}
-\log g(t) \neq (1-t)\,\bar{\lambda}(t).
\end{equation}
The left-hand side summarizes collective spectral mass after aggregation; the 
right-hand side averages directional growth rates. The two quantities are 
numerically distinct and sensitive to different aspects of the transport 
geometry.

\subsection{Illustrative Example}

The distinction between the mean FTLE and the FTSS is not merely formal; it 
has quantitative consequences. Consider a generative flow where, at some time 
$t$, half of the trajectories encounter an anisotropic transport operator 
$\Phi_A$ and the other half encounter $\Phi_B$, defined as:
\begin{equation}
\Phi_A = \operatorname{diag}(10,\,1,\,1,\,\ldots,\,1), \qquad
\Phi_B = \operatorname{diag}(10,\,0.1,\,0.1,\,\ldots,\,0.1),
\end{equation}
in a high-dimensional space ($d \gg 1$). Both operators share the same 
dominant singular value $\sigma_{\max}=10$, so the dominant FTLE is identical 
for all trajectories:
\begin{equation}
\lambda_{\max} = \frac{1}{1-t}\log 10.
\end{equation}

Now compare the two aggregate summaries. The mean FTLE $\bar{\lambda}(t)$, 
which averages $\log\sigma_i$ over directions and trajectories, gives:
\begin{equation}
(1-t)\,\bar{\lambda}(t) 
= \frac{1}{d}\;\mathbb{E}_{x_t}\!\left[\sum_{i=1}^{d} \log\sigma_i\right]
\approx \frac{1}{2}\cdot\frac{\log 10 + (d-1)\log 1}{d} 
     + \frac{1}{2}\cdot\frac{\log 10 + (d-1)\log 0.1}{d}.
\end{equation}
For large $d$, the $d-1$ contracting directions in $\Phi_B$ dominate. 
Specifically, $\log 1 = 0$, so the $\Phi_A$ term becomes $\frac{1}{2}\cdot\frac{\log 10}{d} \approx 0$, 
and the $\Phi_B$ term becomes $\frac{1}{2}\cdot\frac{\log 10 + (d-1)\log 0.1}{d} \approx \frac{1}{2}\log 0.1 \approx -1.15$.
Thus $(1-t)\,\bar{\lambda}(t) \approx -1.15$, indicating strong average contraction 
on the logarithmic scale.

In contrast, the FTSS first computes the mean squared singular value under 
the expectation:
\begin{equation}
g(t)^2 = \frac{1}{2}\cdot\frac{100 + (d-1)\cdot 1}{d} 
      + \frac{1}{2}\cdot\frac{100 + (d-1)\cdot 0.01}{d}.
\end{equation}
For large $d$, $g_A^2 \approx 1$ and $g_B^2 \approx 0.01$, so
\begin{equation}
g(t)^2 \approx \frac{1}{2}(1 + 0.01) = 0.505, \qquad 
g(t) \approx 0.71, \qquad 
-\log g(t) \approx 0.34.
\end{equation}

The two summaries paint different pictures. The mean FTLE is dominated by the 
many strongly contracting directions in $\Phi_B$, reporting substantial 
compression. The FTSS, by squaring the singular values before averaging, 
gives more weight to the single amplified direction ($\sigma=10$) and reports 
only mild overall attenuation ($g \approx 0.71$, close to the identity value 
of $1.0$). Neither summary is more correct; they measure different aspects of 
the transport geometry.

\subsection{Computational Considerations}

Isolating $\lambda_{\max}(t)$ requires iterative tracking of the dominant 
singular value, typically demanding power iteration, tangent-space 
integration, or repetitive Jacobian-vector products across specialized 
trajectories~\cite{benettin1980lyapunov}. Computing the full mean 
log-singular-value would similarly require explicit spectral decomposition. 
By contrast, the FTSS requires no explicit derivative extraction or matrix 
factorization. As formalized in Algorithm~\ref{alg:gain}, it is computed 
directly through forward-pass integration of isotropic finite differences, 
offering a different computational method.

\end{document}